\documentclass[12pt]{article} 
\usepackage[sectionbib]{natbib}
\usepackage{array,epsfig,fancyheadings,rotating}
\usepackage{sectsty, secdot}
\sectionfont{\fontsize{12}{14pt plus.8pt minus .6pt}\selectfont}
\renewcommand{\theequation}{\thesection\arabic{equation}}
\subsectionfont{\fontsize{12}{14pt plus.8pt minus .6pt}\selectfont}

\usepackage{amsmath}
\usepackage{amssymb}
\usepackage{amsfonts}
\usepackage{multirow}
\usepackage{amsthm}
\usepackage{todonotes}
\usepackage{mathtools}
\usepackage{subcaption}
\usepackage{color}
\usepackage{cite}

\newtheorem{theorem}{Theorem}
\newtheorem{lemma}{Lemma}
\newtheorem{corollary}{Corollary}

\theoremstyle{definition}
\newtheorem{definition}{Definition}
\newtheorem{example}{Example}

\newtheorem{condition}{Condition}
\newtheorem{assump}{Assumption}
\DeclareMathOperator{\diag}{diag}
\DeclareMathOperator{\spann}{span}
\DeclareMathOperator{\Diam}{Diam}
\DeclareMathOperator*{\argmin}{arg\,min}

\begin{document}


\renewcommand{\baselinestretch}{1.25}

\markright{ \hbox{\footnotesize\rm Statistica Sinica
}\hfill\\[-13pt]
\hbox{\footnotesize\rm
}\hfill }

\markboth{\hfill{\footnotesize\rm Jean Feng and Noah Simon} \hfill}
{\hfill {\footnotesize\rm Hyper-parameter selection via split-sample validation} \hfill}

\renewcommand{\thefootnote}{}
$\ $\par


\fontsize{12}{14pt plus.8pt minus .6pt}\selectfont \vspace{0.8pc}
\centerline{\large\bf An analysis of the cost of hyper-parameter selection via split-}
\vspace{2pt} \centerline{\large\bf sample validation, with applications to penalized regression}
\vspace{.4cm} \centerline{Jean Feng, Noah Simon} \vspace{.4cm} \centerline{\it
Department of Biostatistics, University of Washington} \vspace{.55cm} \fontsize{9}{11.5pt plus.8pt minus
.6pt}\selectfont


\begin{quotation}
\noindent {\it Abstract:}
In the regression setting, given a set of hyper-parameters, a model-estimation procedure constructs a model from training data. The optimal hyper-parameters that minimize generalization error of the model are usually unknown. In practice they are often estimated using split-sample validation. Up to now, there is an open question regarding how the generalization error of the selected model grows with the number of hyper-parameters to be estimated. To answer this question, we establish finite-sample oracle inequalities for selection based on a single training/test split and based on cross-validation. We show that if the model-estimation procedures are smoothly parameterized by the hyper-parameters, the error incurred from tuning hyper-parameters shrinks at nearly a parametric rate. Hence for semi- and non-parametric model-estimation procedures with a fixed number of hyper-parameters, this additional error is negligible. For parametric model-estimation procedures, adding a hyper-parameter is roughly equivalent to adding a parameter to the model itself. In addition, we specialize these ideas for penalized regression problems with multiple penalty parameters. We establish that the fitted models are Lipschitz in the penalty parameters and thus our oracle inequalities apply. This result encourages development of regularization methods with many penalty parameters.
\vspace{9pt}

\par
\end{quotation}\par

\def\thefigure{\arabic{figure}}
\def\thetable{\arabic{table}}

\renewcommand{\theequation}{\thesection.\arabic{equation}}

\fontsize{12}{14pt plus.8pt minus .6pt}\selectfont

\setcounter{section}{1} 
\setcounter{equation}{0} 

\section{Introduction}

Per the usual regression framework, suppose we observe response $y \in \mathbb{R}$ and predictors $\boldsymbol {x} \in \mathbb{R}^p$. Suppose $y$ is generated by a true model $g^*$ plus random error $\epsilon$ with mean zero, e.g.
$y = g^*(\boldsymbol x) + \epsilon$.
Our goal is to estimate $g^*$.
Many model-estimation procedures can be formulated as selecting a model from some function class $\mathcal{G}$ given training data $T$ and $J$-dimensional hyper-parameter vector $\boldsymbol{\lambda}$. For example, in penalized regression problems, the fitted model can be expressed as the minimizer of the penalized training criterion
\begin{equation}
\label{eq:intro_pen_reg}
\hat{g}(\boldsymbol \lambda | T) = \argmin_{g\in \mathcal{G}} \sum_{(\boldsymbol{x}_i, y_i) \in T} \left (y_i -  g(\boldsymbol{x}_i) \right )^2 + \sum_{j=1}^J \lambda_j P_j(g),
\end{equation}
where $P_j$ are penalty functions and $\lambda_j$ are penalty parameters that serve as hyper-parameters of the model-estimation procedure.

If $\Lambda$ is a set of possible hyper-parameters, the goal is to find a penalty parameter $\boldsymbol{\lambda} \in \Lambda$ that minimizes the expected generalization error
$
\mathbb{E} \left [
\left ( y - \hat{g}(\boldsymbol{\lambda} | T)(\boldsymbol{x}) \right )^2
\right ].
$
Typically one uses a sample-splitting procedure where models are trained on a random partition of the observed data and evaluated on the remaining data.
One then chooses the hyper-parameter $\hat{\boldsymbol{\lambda}}$ that minimize the error on this validation set.
For a more complete review of cross-validation, refer to \citet{arlot2010survey}.

The performance of split-sample validation procedures is typically characterized by an oracle inequality that bounds the generalization error of the expected model selected from the validation set procedure. For $\Lambda$ that are finite, oracle inequalities have been established for a single training/validation split \citep{gyorfi2006distribution} and a general cross-validation framework \citep{van2003unified, van2004asymptotic}. To handle $\Lambda$ over a continuous range, one can use entropy-based approaches \citep{lecue2012oracle}.

The goal of this paper is to characterize the performance of models when the hyper-parameters are tuned by some split-sample validation procedure. We are particularly interested in an open question raised in \citet{bengio2000gradient}: what is the ``amount of overfitting... when too many hyper-parameters are optimized''? In addition, how many hyper-parameters is ``too many''? In this paper we show that actually a large number of hyper-parameters can be tuned without overfitting. In fact, if an oracle estimator converges at rate $R(n)$, then the number of hyper parameters $J$ can grow at roughly a rate of $J = O_p(nR(n))$ up to log terms without affecting the convergence rate. In practice, for penalized regression, this means that one can propose and tune over much more complex models than are currently often used.

To show these results, we prove that finite-sample oracle inequalities of the form
\begin{equation}
\label{thrm:intro_oracle_ineq}
\mathbb{E} \left [
\left ( y - \hat{g}(\hat{\boldsymbol{\lambda}} | T)(\boldsymbol{x}) \right )^2
\right ]
\le
(1+a)
\underbrace{
	\inf_{\lambda \in \Lambda}
	\mathbb{E} \left [
	\left ( y - \hat{g}(\boldsymbol{\lambda} | T)(\boldsymbol{x}) \right )^2
	\right ]
}_{\text{Oracle risk}}
+ \delta \left(J,n\right)
\end{equation}
are satisfied with high probability for some constant $a \ge 0$ and remainder $\delta(J,n)$ that depends on the number of tuned hyper-parameters $J$ and the number of samples $n$.
Under the assumption that the model -estimation procedure is Lipschitz in the hyper-parameters, we find that $\delta$ scales linearly in $J$.
For parametric model-estimation procedures, the additional error from tuning hyper-parameters is roughly $O_p(J/n)$, which is similar to the typical parametric model-estimation rate $O_p(p/n)$ where the model parameters are not regularized.
For semi- and non-parametric model-estimation procedures, this error is generally dominated by the oracle risk so we can actually grow the number of hyper-parameters without affecting the asymptotic convergence rate.

In addition, we specialize our results to penalized regression models of the form \eqref{eq:intro_pen_reg}.
The models in our examples are Lipschitz so that our oracle inequalities apply.
This suggests that multiple penalty parameters may improve the model estimation and that the recent interest in combining penalty functions (e.g. elastic net and sparse group lasso \citep{zou2003regression, simon2013sparse}) may have artificially restricted themselves to two-way combinations.

During our literature search, we found few theoretical results relating the number of hyper-parameters to the generalization error of the selected model. 
Much of the previous work only considered tuning a one-dimensional hyper-parameter over a finite $\Lambda$, proving asymptotic optimality \citep{van2004asymptotic} and finite-sample oracle inequalities \citep{van2003unified, gyorfi2006distribution}. Others have addressed split-sample validation for specific penalized regression problems with a single penalty parameter, such as linear model selection \citep{li1987asymptotic, shao1997asymptotic, golub1979generalized, chetverikov2016cross, chatterjee2015prediction}.
Only the results in \citet{lecue2012oracle} are relevant to answering our question of interest. A potential reason for this dearth of literature is that, historically, tuning multiple hyper-parameters was computationally difficult.
However there have been many recent proposals that address this computational hurdle \citep{bengio2000gradient, foo2008efficient, snoek2012practical}.

Section \ref{sec:main_results} presents oracle inequalities for sample-splitting procedures to understand how the number of hyper-parameters affects the model error.
Section \ref{sec:examples} applies these results to penalized regression models.
Section \ref{sec:simulations} provides a simulation study to support our theoretical results.
Oracle inequalities for general model-estimation procedures and proofs are given in the Supplementary Materials.

\section{Oracle Inequalities} \label{sec:main_results}

Here we establish oracle inequalities for models where the hyper-parameters are tuned by a single training/validation split and cross-validation.
We are interested in studying model-estimation procedures that vary smoothly in their hyper-parameters; such procedures tend to be easier to use and therefore tend to be more popular.

Let $D^{(n)}$ denote a dataset with $n$ samples.
Given dataset training data $D^{(m)}$, let $\hat{g}^{(m)}(\boldsymbol{\lambda} | D^{(m)})$ be some model-estimation procedure that maps hyper-parameter $\boldsymbol{\lambda}$ to a function in $\mathcal{G}$.
We assume the following Lipschitz-like assumption on the model-estimation procedure.
In particular, we suppose that for any $\boldsymbol{x}$, the predicted value $\hat{g}^{(m)}(\boldsymbol{\lambda} | D^{(m)})(\boldsymbol{x})$ is Lipschitz in $\boldsymbol{\lambda}$:
\begin{assump}
	\label{assump:lipschitz}
	Suppose there is a set $\mathcal{X}^{(L)} \subseteq \mathcal{X}$ such that for any $n_T \in \mathbb{N}$ and dataset $D^{(n_T)}$, there is a function $C_\Lambda(\boldsymbol{x} | D^{(n_T)}) : \mathcal{X}^{(L)} \mapsto \mathbb{R}^+$ such that for any $\boldsymbol{x} \in \mathcal{X}^{(L)}$, we have for all $\boldsymbol{\lambda}^{(1)}, \boldsymbol{\lambda}^{(2)} \in \Lambda$
	\begin{align}
	\left |
	\hat{g}^{(n_T)}(\boldsymbol{\lambda}^{(1)}|D^{(n_T)})(\boldsymbol{x}) - \hat{g}^{(n_T)}(\boldsymbol{\lambda}^{(1)}|D^{(n_T)})(\boldsymbol{x}) \right |
	\le C_\Lambda(\boldsymbol{x}|D^{(n_T)}) \|\boldsymbol{\lambda}^{(1)} - \boldsymbol{\lambda}^{(2)}\|_2.
	\end{align}
\end{assump}
\noindent We provide examples of penalized regression models that satisfy this assumption in Section \ref{sec:examples}.

\subsection{A Single Training/Validation Split}\label{sec:single}

In the training/validation split procedure, the dataset $D^{(n)}$ is randomly partitioned into a training set $T = (X_T, Y_T)$ and validation set $V = (X_V, Y_V)$ with $n_T$ and $n_V$ observations, respectively.
The selected hyper-parameter $\hat{\boldsymbol{\lambda}}$ is a minimizer of the validation loss
\begin{equation}
\label{eq:train_val_lambda}
\hat{\boldsymbol \lambda} \in \argmin_{\boldsymbol{\lambda} \in\Lambda} \frac{1}{2} \left \| y-\hat{g}^{(n_T)}( \boldsymbol \lambda | T) \right \|_{V}^{2}
\end{equation}
where $\| h \|^2_{V} \coloneqq \frac{1}{n_V}\sum_{(x_i, y_i)\in V} h^2(x_i, y_i)$ for function $h$.

We now present a finite-sample oracle inequality for the single training/validation split assuming Assumption~\ref{assump:lipschitz} holds.
Our oracle inequality is sharp, i.e. $a=0$ in \eqref{thrm:intro_oracle_ineq}, unlike most other work \citep{gyorfi2006distribution, lecue2012oracle, van2003unified}.
Note that the result below is a special case of Theorem \ref{thrm:train_val_complicated} in Supplementary Materials \ref{appendix:train_val}, which applies to general model-estimation procedures.
\begin{theorem}
	\label{thrm:train_val}
	Let $\Lambda=[\lambda_{\min},\lambda_{\max}]^{J}$ where $\Delta_{\lambda} = \lambda_{\max} - \lambda_{\min} \ge 0$.
	Suppose random variables $\epsilon_i$ from the validation set $V$ are independent with expectation zero and are uniformly sub-Gaussian with parameters $b$ and $B$:
	$$
	\max_{i: (x_i, y_i) \in V} B^2 \left ( \mathbb{E} e^{|\epsilon_i|^2/B^2} - 1 \right ) \le b^2.
	$$
	Let the oracle risk be denoted
	\begin{equation}
	\tilde{R}(X_V|T) = \argmin_{\lambda \in \Lambda} \left \| g^*-\hat{g}^{(n_T)}( \boldsymbol{\lambda} | T) \right \|_{V}^{2}.
	\label{eq:tilde_lambda_def}
	\end{equation}
	Suppose Assumption~\ref{assump:lipschitz} is satisfied over the set $X_V$.
	Then there is a constant $c>0$ only depending on $b$ and $B$ such that for all $\delta$ satisfying
	\begin{equation}
	\delta^{2}
	\ge
	c \left (
	\frac{J \log(\|C_\Lambda(\cdot |T)\|_V \Delta_{\Lambda} n + 1)}{n_{V}}
	\vee 
	\sqrt{\frac{J \log(\|C_\Lambda(\cdot |T)\|_V \Delta_{\Lambda} n + 1)}{n_{V}}
		\tilde{R}(X_V|T)}
	\right )
	\label{thrm:train_val_delta}
	\end{equation}
	we have
	\begin{align}
	Pr\left(
	\left\Vert g^* - \hat{g}^{(n_T)}( \hat{\boldsymbol{\lambda}} | T) \right\Vert _{V}^2 -
	\tilde{R}(X_V|T)
	\ge\delta^2
	\middle | 
	T, X_V
	\right )
	&\le c\exp\left(-\frac{n_{V}\delta^{4}}{c^{2} \tilde{R}(X_V|T)}\right)
	+ c\exp\left(-\frac{n_{V}\delta^{2}}{c^{2}}\right).
	\end{align}
	
\end{theorem}
\noindent
Theorem \ref{thrm:train_val} states that with high probability, the excess risk, e.g. the error incurred during the hyper-parameter selection process, is no more than $\delta^2$.
As seen in \eqref{thrm:train_val_delta}, $\delta^2$ is the maximum of two terms: a near-parametric term and the geometric mean of the near-parametric term and the oracle risk. To see this more clearly, we express Theorem \ref{thrm:train_val} using asymptotic notation.
\begin{corollary}
	\label{corr:train_val}
	Under the assumptions given in Theorem \ref{thrm:train_val}, we have
	\begin{align}
	\left\Vert g^* - \hat{g}^{(n_T)}( \hat{\boldsymbol{\lambda}} | T) \right\Vert _{V}^2 &
	\le \min_{\lambda \in \Lambda} \left\Vert g^* - \hat{g}^{(n_T)}( {\boldsymbol{\lambda}} | T) \right \Vert^2_{V}
	\label{eq:asym_train_val_oracle_risk}
	\\
	& + O_p \left(\frac{J\log (n \|C_\Lambda\|_V \Delta_{\Lambda} )}{n_{V}} \right) 
	\label{eq:asym_train_val_theorem1} \\
	& + O_p \left(
	\sqrt{
		\frac{J \log (n \|C_\Lambda\|_V \Delta_{\Lambda} )}{n_{V}}
		\min_{\lambda \in \Lambda} \left\Vert g^* - \hat{g}^{(n_T)}( {\boldsymbol{\lambda}} | T) \right \Vert^2_{V}
	}
	\right ).
	\label{eq:asym_train_val_theorem2}
	\end{align}
\end{corollary}
\noindent
Corollary \ref{corr:train_val} show that the risk of the selected model is bounded by the oracle risk, the near-parameteric term \eqref{eq:asym_train_val_theorem1}, and the geometric mean of the two values \eqref{eq:asym_train_val_theorem2}.
We refer to \eqref{eq:asym_train_val_theorem1} as near-parametric because the error term in (un-regularized) parametric regression models is typically $O_p(J/n)$, where $J$ is the parameter dimension and $n$ is the number of training samples. Analogously, \eqref{eq:asym_train_val_theorem1} is $O_p(J/n_V)$ modulo a $\log n$ term in the numerator.
The geometric mean \eqref{eq:asym_train_val_theorem2} can be thought of as a consequence of tuning hyper-parameters over
\begin{equation}
\mathcal{G}(T) = \left \{ \hat{g}^{(n_T)}( {\boldsymbol{\lambda}}| T) : \boldsymbol{\lambda} \in \Lambda \right \}.
\end{equation}
As $\mathcal{G}(T)$ does not (or is very unlikely to) contain the true model $g^*$, tuning the hyper-parameters via training/validation split is tuning over a the misspecified model class.
The geometric mean takes into account this misspecification error.

In the semi- and non-parametric regression settings, the oracle error usually shrinks at a rate of $O_p(n_T^{-\omega})$ where $\omega \in (0, 1)$.
If the number of hyper-parameters is fixed and $n$ is large, the oracle risk will tend to dominate the upper bound.
Hence for such problems, we can actually let the number of hyper-parameters grow -- the asymptotic convergence rate of the upper bound will be unchanged as long as $J$ grows no faster than
$
O_p\left (
\frac{n_{V} n_T^{-\omega}}{\log (n \|C_\Lambda\|_V \Delta_{\Lambda})}
\right ).
$

\subsection{Cross-Validation}\label{sec:cv}

Now we give an oracle inequality for $K$-fold cross-validation.
Previously, the oracle inequality was with respect to the $L_2$-norm over the validation covariates.
Now we give our result with respect to the functional $L_2$-norm.
We suppose our dataset is composed of independent identically distributed observations $(X,y)$ where $X$ is independent of $\epsilon$.
The functional $L_2$-norm is defined as
$
\left \| h \right \|^2_{L_2} = \int \left |h(x) \right |^2 d\mu(x)
$.

For $K$-fold cross-validation, we randomly partition the dataset $D^{(n)}$ into $K$ sets, which we assume to have equal size for simplicity. Partition $k$ will be denoted $D_k^{(n_V)}$ and its complement will be denoted $D_{-k}^{(n_T)} = D^{(n)} \setminus D_k^{(n_V)}$. We train our model using $D_{-k}^{(n_T)}$ for $k=1,...,K$ and select the hyper-parameter that minimizes the average validation loss
\begin{eqnarray}
\label{kfold_opt}
\hat{\boldsymbol \lambda} &=& \argmin_{\boldsymbol{\lambda} \in\Lambda} \frac{1}{K} \sum_{k=1}^K  \left \| y-\hat{g}^{(n_T)}(\boldsymbol \lambda | D_{-k}^{(n_T)}) \right \|_{D_k^{(n_V)}}^{2}.
\end{eqnarray}

In traditional cross-validation, the final model is retrained on all the data with $\hat{\boldsymbol{\lambda}}$. However bounding the generalization error of the retrained model requires additional regularity assumptions \citep{lecue2012oracle}. We consider the ``averaged version of $K$-fold cross-validation'' instead
\begin{equation}
\label{thrm:avg_cv}
\bar{g}\left ( {D^{(n)}} \right ) = 
\frac{1}{K} \sum_{k=1}^K 
\hat{g}^{(n_T)} \left (\hat{\boldsymbol \lambda} \middle | D^{(n_T)}_{-k} \right ).
\end{equation}
To bound the generalization error of \eqref{thrm:avg_cv}, we require an assumption in \citet{lecue2012oracle} that controls the tail behavior of the fitted models.
A classical approach for bounding the tail behavior of random variable $X$ is to bound its Orlicz norm 
$\|X\|_{L_{\psi_1}}= \inf \{C > 0: \mathbb{E}\exp(|X|/C) - 1 \le 1\}$ \citep{van1996weak}.
\begin{assump}
	\label{assump:tail_margin}
	There exist constants $K_0, K_1 \ge 0$ and $\kappa \ge 1$ such that for any $n_T \in \mathbb{N}$, dataset $D^{(n_T)}$, and $\boldsymbol{\lambda} \in \Lambda$, we have
	\begin{align}
	\left \|
	\left(
	y - \hat{g}^{(n_T)}(\boldsymbol{\lambda} | D^{(n_T)})
	\right)^2
	- \left(
	y - g^*
	\right)^2
	\right \|_{L_{\psi_1}} & \le K_0
	\label{eq:cv_assump1}\\
	\left \|
	\left(
	y - \hat{g}^{(n_T)}(\boldsymbol{\lambda} | D^{(n_T)})
	\right)^2
	- \left(
	y - g^*
	\right)^2
	\right \|_{L_2}
	& \le 
	K_1 \left \|
	g^{*}-\hat{g}(\boldsymbol{\lambda}|D^{(n_{T})})
	\right \|_{L_{2}}^{1/\kappa}.
	\label{eq:cv_assump2}
	\end{align}
\end{assump}

With the above assumption, the following oracle inequality bounds the risk of averaged version of $K$-fold cross-validation.
It is a special case of Theorem~\ref{thrm:jean_cv} in the Supplementary Materials, which extends Theorem 3.5 in \citet{lecue2012oracle}.
The notation $\mathbb{E}_{D^{(m)}}$ indicates the expectation over random $m$-sample datasets $D^{(m)}$ drawn from the probability distribution $\mu$.
\begin{theorem}
	\label{thrm:kfold}
	Let $\Lambda=[\lambda_{\min},\lambda_{\max}]^{J}$ where $\Delta_{\Lambda} = (\lambda_{\max} - \lambda_{\min}) \vee 1$.
	Suppose random variables $\epsilon_i$ are independent with expectation zero, satisfy $\|\epsilon\|_{L_{\psi_2}}= b <\infty$, and are independent of $X$.
	Suppose Assumption~\ref{assump:lipschitz} holds over the set $\mathcal{X}$ and Assumption~\ref{assump:tail_margin} holds.
	Suppose there exists a function $\tilde{h}$ and some $\sigma_0 > 0$ such that
	\begin{align}
	\tilde{h}(n_{T})
	\ge
	1 + \sum_{k=1}^{\infty}
	k\Pr\left(\|C_\Lambda(\cdot |D^{(n_{T})})\|_{L_{\psi_{2}}}\ge2^{k}\sigma_{0}\right).
	\label{eq:prob_bound_cv}
	\end{align}
	Then there exists an absolute constant $c_{1}>0$ and a constant $c_{K_0, b}>0$ such that for any $a > 0$,
	\begin{align}
	\begin{split}
	\mathbb{E}_{D^{(n)}}\left(
	\|
	\bar{g}(D^{(n)})
	-g^{*}
	\|_{L_{2}}^{2}\right)
	& \le	(1+a)
	\inf_{\lambda\in\Lambda}
	\left[\mathbb{E}_{D^{(n_{T})}}\left(\|
	\hat{g}(\boldsymbol{\lambda}|D^{(n_{T})})
	-g^{*}\|_{L_{2}}^{2}\right)\right] \\
	& +
	c_{1}
	\left (\frac{1+a}{a} \right )^2
	\frac{J\log n_{V}}{n_{V}}
	K_0
	\left[\log\left(\Delta_{\Lambda} c_{K_0, b} n \sigma_0 +1\right)+1\right]
	\tilde{h}(n_{T}).
	\end{split}
	\label{eq:cv_lipschitz_oracle_ineq}
	\end{align}
\end{theorem}

As in Theorem \ref{thrm:train_val}, the remainder term in Theorem~\ref{thrm:kfold} includes a near-parametric term $O_p(J/n_V)$.
So as before, adding hyper-parameters to parametric model estimation incurs a similar cost as adding parameters to the parametric model itself and adding hyper-parameters to semi- and non-parametric regression settings is relatively ``cheap" and negligible asymptotically.

The differences between Theorems \ref{thrm:train_val} and \ref{thrm:kfold} highlight the tradeoffs made to establish an oracle inequality involving the functional $L_2$-error.
The biggest tradeoff is that Theorem~\ref{thrm:kfold} adds Assumption~\ref{assump:tail_margin}.
Though we can relax Assumption~\ref{assump:tail_margin} to hold over datasets $D$ in some high-probability set, the difficulty lies in controlling the tail behavior of the fitted models over all $\Lambda$.
For some model estimation procedures, $K_0$ may grow with $n$ if $\lambda_{\min}$ shrinks too quickly with $n$.
In this case, the remainder term may not longer shrink at a near-parametric rate.
Unfortunately requiring $\lambda_{\min}$ to shrink at an appropriate rate seems to defeat the purpose of cross-validation.
So even though Theorem~\ref{thrm:kfold} helps us better understand cross-validation, it is limited by this assumption.
In addition, the Lipschitz assumption must hold over all $\mathcal{X}$ in Theorem~\ref{thrm:kfold}, rather than just the observed covariates.
Finally, the oracle inequality in Theorem~\ref{thrm:kfold} is no longer sharp since the oracle risk is scaled by $1+a$ for $a > 0$.

\section{Penalized regression models}
\label{sec:examples}
Now we apply our results to analyze penalized regression procedures of the form \eqref{eq:intro_pen_reg}.
Penalty functions encourage particular characteristics in the fitted models (e.g. smoothness or sparsity) and combining multiple penalty functions results in models that exhibit a combination of the desired characteristics. 
There is much interest in combining multiple penalty functions, but few methods incorporate more than two penalties due to (a) the concern that models may overfit the data when selection of many penalty parameters is required; and (b) computational issues in optimizing multiple penalty parameters. In this section, we evaluate the validity of concern (a) using the results of Section~\ref{sec:main_results}. We see that, contrary to popular wisdom, using split-sample validation to select multiple penalty parameters should not result in a drastic increase to the generalization error of the selected model.

In this section, we consider penalty parameter spaces of the form
$\Lambda = [ n^{-t_{\min}}, n^{t_{\max}}]^J$
for $t_{\min}, t_{\max} \ge 0$.
This regime works well for two reasons: one, our rates depend only quite weakly on $t_{\min}$ and $t_{\max}$; and two, oracle $\lambda$-values are generally $O_p(n^{-\alpha})$ for some $\alpha \in (0,1)$ \citep{van2000empirical, van2015penalized, buhlmann2011statistics}. So long as $t_{\min} > \alpha$, $\Lambda$ will contain the optimal penalty parameter.
We do not consider settings where $\lambda_{\min}$ shrinks faster than a polynomial rate since the fitted models can be ill-behaved.

In the following sections, we do an in-depth study of additive models of the form
\begin{equation}
g(\boldsymbol{x}^{(1)}, ..., \boldsymbol{x}^{(J)})= \sum_{j=1}^J g_j(\boldsymbol{x}^{(j)}).
\end{equation}
We first consider parametric additive models (with potentially growing numbers of parameters) fitted with smooth and non-smooth penalties and then nonparametric additive models.
We find that the Lipschitz function $C_\Lambda(\boldsymbol{x} | T)$ scales with $n^{O_p(t_{\min})}$.
Applying Theorems~\ref{thrm:train_val} and \ref{thrm:kfold}, we find that the near-parametric term in the remainder only grows linearly in $t_{\min}$.
We apply these results to various additive model estimation methods.
For instance, in the generalized additive model example, we show that under minimal assumptions, the error from tuning penalty parameters is negligible compared to the error from solving the penalized regression problem with oracle penalty parameters.

\subsection{Parametric additive models}
\label{sec:param_add_models}
Parametric additive models with model parameters $\boldsymbol{\theta} = \left (\boldsymbol{\theta}^{(1)}, ..., \boldsymbol{\theta}^{(J)} \right )$ have the form
\begin{equation}
g(\boldsymbol{\theta})(\boldsymbol{x})
= \sum_{j=1}^J g_j(\boldsymbol{\theta}^{(j)})(\boldsymbol{x}^{(j)}).
\end{equation}
We denote the training criterion for training data $T$ as
\begin{equation}
\label{eq:param_add}
L_T \left (\boldsymbol{\theta}, \boldsymbol{\lambda} \right) 
\coloneqq \frac{1}{2} \left  \| y -  g(\boldsymbol{\theta}) \right \|^2_T 
+ \sum_{j=1}^J \lambda_j P_j(\boldsymbol{\theta}^{(j)}).
\end{equation}
Suppose $\boldsymbol{\theta}^*$ is the unique minimizer of the expected loss $\| y - g(\boldsymbol{\theta}) \|^2_{L_2}$.

\subsubsection{Parametric regression with smooth penalties}
\label{sec:param_smooth}
We begin with the simple case where the penalty functions are smooth. The following lemma states that the fitted models are Lipschitz in the penalty parameter vector.
Given matrices $A$ and $B$, $A \succeq B$ means that $A - B$ is a positive semi-definite matrix.
\begin{lemma}
	\label{lemma:param_add}
	Let $\Lambda \coloneqq \left [ \lambda_{\min}, \lambda_{\max} \right ]^J$ where $\lambda_{\max} \ge \lambda_{\min} > 0$.
	For a fixed training dataset $T \equiv D^{(n_T)}$, suppose for all $\boldsymbol{\lambda} \in \Lambda$, $L_T \left (\boldsymbol{\theta}, \boldsymbol{\lambda} \right)$ has a unique minimizer
	\begin{equation}
	\label{eq:param_add_estimator}
	\left\{
	\hat{\boldsymbol{\theta}}^{(j)}\left (\boldsymbol{\lambda} | T \right )
	\right\}_{j=1}^J =
	\argmin_{\boldsymbol{\theta} \in \mathbb{R}^p} L_T \left (\boldsymbol{\theta}, \boldsymbol{\lambda} \right).
	\end{equation}
	Suppose for all $j = 1,...,J$, the parametric class $g_j$ is $\ell_j$-Lipschitz in its parameters
	\begin{align}
	\left|
	g_j (\boldsymbol{\theta}^{(1)} )(\boldsymbol{x}^{(j)})
	-g_j (\boldsymbol{\theta}^{(2)} )(\boldsymbol{x}^{(j)})
	\right|
	\le
	\ell_j (\boldsymbol{x}^{(j)})
	\|\boldsymbol{\theta}^{(1)}-\boldsymbol{\theta}^{(2)}\|_{2}
	\quad
	\forall \boldsymbol{x}^{(j)} \in \mathcal{X}^{(j)}.
	\label{eq:lipschitz_g}
	\end{align}
	Further suppose for all $j=1,..,J$, $P_j(\boldsymbol{\theta}^{(j)})$ and $g_j(\boldsymbol{\theta}^{(j)})(\boldsymbol{x})$ are twice-differentiable with respect to $\boldsymbol{\theta}^{(j)}$ for any fixed $\boldsymbol{x}$.
	Suppose there exists an $m(T) > 0$ such that the Hessian of the penalized training criterion at the minimizer satisfies
	\begin{equation}
	\left . \nabla_{\theta}^2 L_T \left (\boldsymbol{\theta}, \boldsymbol{\lambda} \right) \right |_{\theta = \hat{\theta}(\boldsymbol{\lambda} | T )} \succeq m(T) \boldsymbol{I}
	\quad \forall \boldsymbol{\lambda} \in \Lambda,
	\label{eq:smooth_pos_def}
	\end{equation}
	where $I$ is a $p \times p$ identity matrix.
	Then for any $\boldsymbol{\lambda}^{(1)}, \boldsymbol{\lambda}^{(2)} \in \Lambda$,
	Assumption~\ref{assump:lipschitz} is satisfied over the set $\mathcal{X}^{(1)} \times ... \times \mathcal{X}^{(J)}$ with function
	\begin{equation}
	\label{eq:param_add_lipschitz}
	C_\Lambda(\boldsymbol{x} | T) =
	\frac{1}{m(T) \lambda_{min}}
	\sqrt{
		\left(
		\left\Vert \epsilon \right \Vert_T^2 + 2 C^*_{\Lambda}
		\right)
		\left(
		\sum_{j=1}^J  \|\ell_j\|_T^2 \ell_j^2(\boldsymbol{x}^{(j)})
		\right)
	}
	\end{equation}
	where $C^*_{\Lambda} = \lambda_{max}\sum_{j=1}^{J} P_{j}(\boldsymbol{\theta}^{(j),*})$.
\end{lemma}
\noindent Notice that Lemma~\ref{lemma:param_add} requires the training criterion to be strongly convex at its minimizer.
This is satisfied in the following example involving multiple ridge penalties.
If \eqref{eq:smooth_pos_def} is not satisfied by a penalized regression problem, one can consider a variant of the problem where the penalty functions $P_j(\boldsymbol{\theta}^{(j)})$ are replaced with penalty functions
$P_j(\boldsymbol{\theta}^{(j)}) + \frac{w}{2}\| \boldsymbol{\theta}^{(j)} \|_2^2$ for a fixed $w > 0$.
\begin{example}[Multiple ridge penalties]
	\label{ex:ridge}
	Let us consider fitting a linear model via ridge regression.
	If we can group covariates based on the similarity of their effects on the response, e.g. $\boldsymbol{x} = (\boldsymbol{x}^{(1)}, ... , \boldsymbol{x}^{(J)})$ where $\boldsymbol{x}^{(j)}$ is a vector of length $p_j$, we can incorporate this prior information by penalizing each group of covariates differently:
	\begin{align}
	L_T \left (\boldsymbol{\theta}, \boldsymbol{\lambda} \right) 
	\coloneqq 
	\frac{1}{2}
	\left\|
	y -  \sum_{j=1}^J \boldsymbol{x}^{(j)} \boldsymbol{\theta}^{(j)}
	\right \|^2_T
	+ \sum_{j=1}^J \frac{\lambda_j}{2} \|\boldsymbol{\theta}^{(j)}\|_2^2.
	\end{align}
	We tune the penalty parameters $\boldsymbol{\lambda}$ over the set $\Lambda$ via a training/validation split with training and validation sets $T$ and $V$, respectively.
	For all the examples in this manuscript, let $\Lambda = \left [n^{- t_{\min}}, 1 \right ]^J$.

	Via some algebra, we can derive \eqref{eq:param_add_lipschitz} in Lemma~\ref{lemma:param_add}; the details are deferred to the Supplementary Materials.
	Plugging this result into Corollary~\ref{corr:train_val}, we find that the parametric term \eqref{eq:asym_train_val_theorem1} in the remainder is on the order of
	\begin{align}
	\frac{J t_{\min}}{n_{V}}
	\log \left (
	C^*_T n
	\sum_{j = 1}^J
	\left(\frac{1}{n_T} \sum_{(x_i, y_i) \in T} \|\boldsymbol{x}_i^{(j)}\|^2_2\right)
	\left(\frac{1}{n_V} \sum_{(x_i, y_i) \in V} \|\boldsymbol{x}_i^{(j)}\|^2_2\right)
	\right )
	\end{align}
	where $
	C^*_{T} =
	\|\epsilon\|_{T}^{2}
	+ \sum_{j=1}^J \|\boldsymbol{\theta}^{*,(j)}\|_2^2
	.$
	So we have shown in this example that if the lower bound of $\Lambda$ shrinks at the polynomial rate $n^{-t_{\min}}$, the near-parametric term in the remainder of the oracle inequality grows only linearly in its power $t_{\min}$.
\end{example}

In the next example, we consider generalized additive models (GAMs) \citep{hastie1990generalized}.
Though GAMs are nonparametric models, it is well-known that they are equivalent to solving a finite-dimensional problem \citep{green1993nonparametric, o1986automatic, buja1989linear}.
By reformulating GAMs as parametric models instead, we can establish oracle inequalities for tuning the penalty parameters via training/validation split.
Here we present an outline of the procedure; the details can be found in the Supplementary Materials.
\begin{example}[Multiple sobolev penalties]
	\label{example:sobolev}
	To fit a generalized additive model over the domain $\mathcal{X}^J$ where $\mathcal{X} \subseteq \mathbb{R}$, a typical setup is to solve
	\begin{align}
	\argmin_{\alpha_0 \in \mathbb{R}, g_j}
	\frac{1}{2} \sum_{i\in D^{(n_T)}}
	\left(
	y_i - \alpha_0 - \sum_{j=1}^J g_j(x_{ij})
	\right)^2
	+ \sum_{j=1}^{J} \lambda_j \int_{\mathcal{X}} \left(g_j^{''}(x_j)\right)^{2} dx_j
	\label{eq:smoothing_spline}
	\end{align}
	where the penalty function is the 2nd-order Sobolev norm.
	Let $\mathcal{X} = [0,1]$ for this example.
	Using properties of the Sobolev penalty, \eqref{eq:smoothing_spline} can be re-expressed as a finite-dimensional problem with matrices $K_j$
	\begin{align}
	\argmin_{\alpha_0, \boldsymbol{\alpha}_1, \boldsymbol{\theta}}
	\frac{1}{2}
	\left \|
	y -
	\alpha_0 \boldsymbol{1}
	- \boldsymbol{x} \boldsymbol{\alpha}_1
	- \sum_{j=1}^j K_j \boldsymbol{\theta}^{(j)}
	\right \|^2_T
	+
	\frac{1}{2}
	\sum_{j = 1}^J
	\lambda_j \boldsymbol{\theta}^{(j)\top} K_j \boldsymbol{\theta}^{(j)}.
	\label{eq:matrix_sobolev}
	\end{align}
	Let $X_T \in \mathbb{R}^{n_T\times J}$ be the covariates $\boldsymbol{x}$ in the training data stacked together.
	If $X_T^\top X_T$ is invertible, we can derive the closed-form solution for \eqref{eq:matrix_sobolev}.
	From there, we can directly calculate \eqref{eq:param_add_lipschitz} in Lemma~\ref{lemma:param_add}.
	Plugging this result into Corollary~\ref{corr:train_val}, we find that the parametric term in the remainder is on the order of
	\begin{align}
	\frac{J t_{\min}}{n_{V}} \log \left (
	n
	J
	\|y\|_T
	\left(
	J
	\left \|
	\left(
	X_T^\top X_T
	\right)^{-1}
	X_T^\top
	\right \|_2
	+
	\sum_{j=1}^J h_j^{-2}(T)
	\right)
	\right )
	\label{eq:sobolev_param}
	\end{align}
	where $\|\cdot\|_2$ is the spectral norm and $h_j(T)$ is the smallest distance between observations of the $j$th covariates in the training data $T$.

	In particular, for $J = o(n^{1/2})$, the smoothing spline estimate~\eqref{eq:smoothing_spline} is shown to attain the minimax optimal rate of $O_p(Jn^{-4/5})$ if the penalty parameters shrink at the rate of $\sim n^{-4/5}$ \citep{sadhanala2017additive, horowitz2006optimal}.
	From Corollary~\ref{corr:train_val}, we see that the oracle error \eqref{eq:asym_train_val_oracle_risk} asymptotically dominates the additional error terms incurred from tuning the penalty parameters.
	Moreover, as long as we choose $\lambda_{\min} \sim n^{-\alpha}$ for any $\alpha > 4/5$, the model selected via training/validation split will also attain the minimax rate.
\end{example}

\subsubsection{Parametric regression with non-smooth penalties}
\label{sec:param_nonsmooth}

If the penalty functions are non-smooth, similar results do not necessarily hold. Nonetheless we find that for many popular non-smooth penalty functions, such as the lasso \citep{tibshirani1996regression} and group lasso \citep{yuan2006model}, the fitted functions are still smoothly parameterized by $\boldsymbol \lambda$ almost everywhere.
To characterize such problems, we begin with the following definitions from \citet{feng2017gradient}:

\begin{definition}
	The differentiable space of function $f:\mathbb{R}^p \mapsto \mathbb{R}$ at $\boldsymbol{\theta}$ is
	\begin{equation}
	\Omega^{f}(\boldsymbol{\theta}) = \left \{ \boldsymbol{\beta} \middle | \lim_{\epsilon \rightarrow 0} \frac{f(\boldsymbol{\theta} + \epsilon \boldsymbol{\beta}) - f(\boldsymbol{\theta})}{\epsilon} \text{ exists } \right \}.
	\end{equation}
\end{definition}

\begin{definition}
	Let $f(\cdot, \cdot): \mathbb{R}^p \times \mathbb{R}^J \mapsto \mathbb{R}$ be a function with a unique minimizer.
	$S \subseteq \mathbb{R}^p$ is a local optimality space of $f$ over $W \subseteq \mathbb{R}^J$ if
	\begin{equation}
	\argmin_{\boldsymbol{\theta} \in \mathbb{R}^p} f(\boldsymbol{\theta}, \boldsymbol \lambda) =
	\argmin_{\boldsymbol{\theta} \in S} f(\boldsymbol{\theta}, \boldsymbol \lambda) \quad \forall \boldsymbol \lambda \in W.
	\end{equation}
\end{definition}
\noindent Using the definitions above, we can characterize the penalty parameters $\Lambda_{smooth} \subseteq \Lambda$ where the fitted functions are well-behaved.
\begin{condition}
	\label{condn:nonsmooth1}
	For every $\boldsymbol{\lambda} \in \Lambda_{smooth}$, there exists a ball $B(\boldsymbol{\lambda})$ with nonzero radius centered at $\boldsymbol{\lambda}$ such that
	\begin{itemize}
		\item For all $\boldsymbol{\lambda}'\in B(\boldsymbol{\lambda})$, the training criterion $L_{T}(\cdot, \boldsymbol{\lambda}')$ is twice differentiable with respect to $\boldsymbol{\theta}$ at $\hat{\boldsymbol{\theta}}(\boldsymbol{\lambda}'|T)$
		along directions in the product space
		\begin{align}
		\Omega^{L_T(\cdot, \boldsymbol{\lambda})} \left (\hat{\boldsymbol \theta}\left(\boldsymbol{\lambda}|T \right) \right) =
		& \Omega^{P_1(\cdot)}
			\left(\hat{\boldsymbol{\theta}}^{(1)}(\boldsymbol{\lambda} | T)\right)
		\times
		...
		\times
		\Omega^{P_J(\cdot)}
		\left(\hat{\boldsymbol{\theta}}^{(J)}(\boldsymbol{\lambda} | T)\right)
		.
		\end{align}
		\item $\Omega^{L_T(\cdot, \boldsymbol{\lambda})} \left (\hat{\boldsymbol \theta}\left(\boldsymbol{\lambda}|T \right) \right)$ is a local optimality space for $L_T\left(\cdot,\boldsymbol{\lambda}\right)$ over $B(\boldsymbol{\lambda})$.
	\end{itemize}
\end{condition}
\noindent In addition, we need nearly all penalty parameters to be in $\Lambda_{smooth}$.
\begin{condition}
	\label{condn:nonsmooth2}
	$\Lambda \setminus \Lambda_{smooth}$ has Lebesgue measure zero, e.g. $\mu(\Lambda_{smooth}^c) = 0$.
\end{condition}
\noindent For instance, in the lasso, $\Lambda_{smooth}$ is the sections of the lasso-path in between the knots.
As the knots in the lasso-path are countable, the set outside $\Lambda_{smooth}$ has measure zero.

Assuming the above conditions hold, the fitted models for non-smooth penalty functions satisfy the same Lipschitz relation as that in Lemma \ref{lemma:param_add}.

\begin{lemma}
	\label{lemma:nonsmooth}
	Let $\Lambda \coloneqq \left [ \lambda_{\min}, \lambda_{\max} \right ]^J$ where $\lambda_{\max} \ge \lambda_{\min} > 0$.
	Suppose that for all $j = 1,...,J$, $g_j$ satisfies \eqref{eq:lipschitz_g} over $\mathcal{X}^{(j)}$.
	Suppose for training data $T\equiv D^{(n_{T})}$, the penalized loss function $L_{T}\left(\boldsymbol{\theta},\boldsymbol{\lambda}\right)$
	has a unique minimizer $\hat{\boldsymbol{\theta}}(\boldsymbol{\lambda}|T)$
	for every $\boldsymbol{\lambda}\in\Lambda$.
	Let $\boldsymbol{U}_{\lambda}$
	be an orthonormal matrix with columns forming a basis for the differentiable
	space of $L_{T}(\cdot,\boldsymbol{\lambda})$ at $\hat{\boldsymbol{\theta}}(\boldsymbol{\lambda}|T)$.
	Suppose there exists a constant $m(T)>0$ such that the Hessian of
	the penalized training criterion at the minimizer taken with respect
	to the directions in $\boldsymbol{U}_{\lambda}$ satisfies 
	\begin{equation}
	\left._{U_{\lambda}}\nabla_{\theta}^{2}L_{T}(\boldsymbol{\theta},\boldsymbol{\lambda})\right|_{\theta=\hat{\theta}(\boldsymbol{\lambda})}\succeq m(T)\boldsymbol{I}\quad\forall\boldsymbol{\lambda}\in\Lambda
	\end{equation}
	where \textup{$\boldsymbol{I}$ is the identity matrix.}
	Suppose Conditions~\ref{condn:nonsmooth1} and \ref{condn:nonsmooth2} are satisfied.
	Then any $\boldsymbol{\lambda}^{(1)}, \boldsymbol{\lambda}^{(2)} \in \Lambda$ satisfies Assumption~\ref{assump:lipschitz} over $\mathcal{X}^{(1)} \times ... \times \mathcal{X}^{(J)}$ with $C_\Lambda$ defined in \eqref{eq:param_add_lipschitz}.
\end{lemma}

\noindent As an example, we consider multiple elastic net penalties where the penalty parameters are tuned by training/validation split and cross-validation.

\begin{example}[Multiple elastic nets, training/validation split]
	\label{ex:elastic_net_tv}
	Suppose we would like to fit a linear model via the elastic net.
	If the covariates are grouped a priori, we can penalize each group differently using the following objective
	\begin{align}
	\hat{\boldsymbol{\theta}}(\boldsymbol{\lambda})
	=\argmin_{\theta^{(j)} \in \mathbb{R}^{p_j}, j = 1,...,J}
	\frac{1}{2} \left \| y - \sum_{j=1}^J \boldsymbol{X}^{(j)} \boldsymbol{\theta}^{(j)} \right \|_T^2
	+ \sum_{j=1}^J \lambda_j \left(
	\| \boldsymbol{\theta}^{(j)}\|_1
	+ \frac{w}{2} \| \boldsymbol{\theta}^{(j)}\|_2^2
	\right)
	\label{eq:elastic_net_ex}
	\end{align}
	where $w > 0$ is a fixed constant.
	Here we briefly sketch the process for deriving the oracle inequality when the penalty parameters via training/validation split over $\Lambda = [n^{-t_{\min}}, 1]^J$.
	Details are given in Supplementary Materials.

	First we check that all the conditions are satisfied.
	For this problem, the differentiable space is the subspace spanned by the non-zero elements in $\hat{\boldsymbol{\theta}}(\boldsymbol{\lambda})$.
	Since the elastic net solution paths are piecewise linear \citep{zou2003regression}, the differentiable space is also a local optimality space.
	Then using a similar procedure as in Example~\ref{ex:ridge}, we find that the parametric term in the remainder of Corollary~\ref{corr:train_val} is on the order of
	\begin{align}
	\frac{J t_{\min}}{n_{V}}
	\log \left (
	\frac{C^*_T n}{w}
	\sum_{j=1}^J
	\left(\frac{1}{n_T} \sum_{(x_i, y_i) \in T} \|\boldsymbol{x}_i^{(j)}\|^2_2\right)
	\left(\frac{1}{n_V} \sum_{(x_i, y_i) \in V} \|\boldsymbol{x}_i^{(j)}\|^2_2\right)
	\right )
	\label{eq:elastic_net_tv_param_error}
	\end{align}
	where
	$
	C^*_T =
	\|\epsilon\|_{T}^{2}
	+\sum_{j=1}^J
	2 \|\boldsymbol{\theta}^{*,(j)}\|_1
	+ w\|\boldsymbol{\theta}^{*,(j)}\|_2^2
	$.

	We can compare this additional error term to the risk of using an oracle penalty parameter.
	For the case of a single penalty parameter ($J = 1$), the convergence rate of using an oracle penalty parameter for the elastic net is on the order of $O_p(\log(p)/n)$ \citep{bunea2008honest, hebiri2011smooth}.
	If we split the covariates into groups and tune the penalty parameters via training/validation split, the incurred error \eqref{eq:elastic_net_tv_param_error} is on a similar order.
\end{example}
\begin{example}[Multiple elastic nets, cross-validation]
	\label{eq:elastic_net_cv}
	Now we establish an oracle inequality for the averaged version of $K$-fold cross-validation using a similar setup as \citet{lecue2012oracle}.
	Suppose the noise $\epsilon$ is sub-gaussian and for simplicity, suppose $X$ is drawn uniformly from $[-1, 1]^p$.
	In order to satisfy the assumptions in Theorem~\ref{thrm:kfold}, our fitting procedure for $\hat{\boldsymbol{\theta}}(\boldsymbol{\lambda})$ entails a thresholding operation similar to that in \citet{lecue2012oracle}.
	In particular, we fit parameters $\hat{\boldsymbol{\theta}}_{thres}(\boldsymbol{\lambda})$ where the $i$-th element is
	\begin{align}
	\hat{{\theta}}_{thres, i}(\boldsymbol{\lambda})
	= \text{sign}(\hat{{\theta}}_{i}(\boldsymbol{\lambda}))
	(|\hat{{\theta}}_{i}(\boldsymbol{\lambda})| \wedge K_0')
	\quad i = 1,...,p
	\label{eq:threshold_elastic_net}
	\end{align}
	where $\hat{\boldsymbol{\theta}}(\boldsymbol{\lambda})$ is the solution to \eqref{eq:elastic_net_ex} and $K_0' > 0$ is some fixed constant.
	We then find the Lipschitz factor in Lemma~\ref{lemma:nonparam_smooth} and bound its Orlicz norm via exponential concentration inequalities.
	Let $\bar{\boldsymbol{\theta}}(D^{(n)})$ be the fitted parameters using the averaged version of $K$-fold cross-validation.
	By Theorem~\ref{thrm:kfold}, there is some constant $\tilde{c} > 0$, such that for any $a > 0$
	\begin{align}
	\begin{split}
	\mathbb{P}_{D^{(n)}}
	\left \|
	X \left(
	\bar{\boldsymbol{\theta}}(D^{(n)})
	- \boldsymbol{\theta}^*
	\right)
	\right \|_{L_{2}}^{2}
	& \le	(1+a)
	\inf_{\lambda\in\Lambda}
	\left[
	\mathbb{P}_{D^{(n_{T})}}
	\left \|
	X \left(
	\bar{\boldsymbol{\theta}}(D^{(n_T)})
	- \boldsymbol{\theta}^*
	\right)
	\right \|_{L_{2}}^{2}
	\right] \\
	& \quad +
	\tilde{c}
	\left (\frac{1+a}{a} \right )^2
	\frac{J \log n_{V}}{n_{V}}
	t_{\min}
	\log\left(
	\frac{1+a}{aw} Jpn
	\right).
	\end{split}
	\end{align}

\noindent The above example is similar to the lasso example in \citet{lecue2012oracle}; the major difference is that we consider the case where the penalty parameters are tuned over a continuous range.
We are able to do this since Lemma~\ref{lemma:nonsmooth} specifies a Lipschitz relation between the fitted functions and the penalty parameters.
This result is relevant when $J$ is large and $\boldsymbol{\lambda}$ must be tuned via a continuous optimization procedure.
\end{example}

\subsection{Nonparametric additive models}
\label{sec:nonparam_smooth}

We now consider nonparametric additive models of the form
\begin{align}
\label{eq:train_crit_nonparam}
\left\{ \hat{g}_j( \boldsymbol \lambda) \right \}_{j=1}^J
=
\argmin_{g_j\in \mathcal{G}_j: j=1,...,J}  L_T\left (\left \{ g_j \right \}_{j=1}^J, \boldsymbol{\lambda} \right )
\coloneqq
\frac{1}{2} \left \| y -  \sum_{j=1}^J g_j(x_j) \right \|^2_T 
+ \sum_{j=1}^J \lambda_j P_j(g_j)
\end{align}
where $\{P_j\}$ are penalty functionals and $\{\mathcal{G}_j\}$ are linear spaces of univariate functions.
Let $\left\{ g_j^* \right \}_{j=1}^J$ be the minimizer of the generalization error
\begin{equation}
\left\{ g_j^* \right \}_{j=1}^J = \argmin_{g_j \in \mathcal{G}_j: j=1,...,J}
E \left \| y - \sum_{j=1}^J g_j^* \right \|^2_{L_2}.
\end{equation}
We obtain a similar Lipschitz relation in the nonparametric setting to those before.
\begin{lemma}
	\label{lemma:nonparam_smooth}
	Let $\lambda_{\max} > \lambda_{\min} > 0 $ and $\Lambda \coloneqq [\lambda_{\min}, \lambda_{\max}]^J$.
	Suppose the penalty functions $P_{j}$ are twice Gateaux differentiable and convex over $\mathcal{G}_j$.
	Suppose there is a $m(T) > 0$ such that the second Gateaux derivative of the training criterion at $\{\hat{g}^{(n_T)}_j( \boldsymbol{\lambda} | T)\}$ for all $\boldsymbol{\lambda} \in \Lambda$ satisfies
	\begin{align}
	\left \langle 
	\left . D^2_{\{g_j\}} L_T \left ( \left \{ g_j \right \}_{j=1}^J, \boldsymbol{\lambda} \right ) \right |_{g_j= \hat{g}_j( \boldsymbol{\lambda} | T) }
	\circ h_j, h_j
	\right \rangle 
	\ge m(T)
	\quad \forall h_j \in \mathcal{G}_j,  \|h_j \|_{D^{(n)}} = 1
	\label{eq:gateuax}
	\end{align}
	where $D^2_{\{g_j\}}$ is the second Gateaux derivative taken in directions $\{g_j\}$.
	Let $
	C_{\Lambda}^*= \lambda_{max}\sum_{j=1}^{J} P_{j}(g^*_j).
	$
	For any $\boldsymbol{\lambda}^{(1)}, \boldsymbol{\lambda}^{(2)} \in \Lambda$, we have
	\begin{align}
	\label{eq:nonparam_lipshitz_thrm}
	\left\Vert 
	\sum_{j=1}^J \hat{g}_j\left(\boldsymbol{\lambda}^{(1)} |T \right)-\hat{g}_j\left(\boldsymbol{\lambda}^{(2)} |T \right)\right\Vert _{D^{(n)}} & \le
	\frac{m(T)}{\lambda_{min}}
	\sqrt{
		\left(
		\|\epsilon\|_T^2 + 2 C^*_\Lambda
		\right)
		\frac{n_{D}}{n_{T}}
	}
	\left \|\boldsymbol{\lambda}^{(1)}-\boldsymbol{\lambda}^{(2)} \right \|_2.
	\end{align}
\end{lemma}
\noindent
A simple example that satisfies \eqref{eq:gateuax} is a penalized regression model where we fit values at each of the observed covariates, e.g. $\hat{\boldsymbol{\theta}} \in \mathbb{R}^n$, and penalize this fitted value by a ridge penalty.
Note that such a penalty is allowed because the response $y$ in the validation set is not used by the training procedure.

Note that since Lemma~\ref{lemma:nonparam_smooth} verifies that Assumption~\ref{assump:lipschitz} is satisfied over the observed covariates, it is suitable to be used in Theorem~\ref{thrm:train_val}.
However \eqref{eq:nonparam_lipshitz_thrm} is not a strong enough statement to be used for Theorem~\ref{thrm:kfold}.

\section{Simulations}\label{sec:simulations}

We now present a simulation study of the generalized additive model in Example \ref{example:sobolev} to understand how the performance changes as the number of penalty parameters $J$ increases.
Corollary~\ref{corr:train_val} suggests that there are two opposing forces that affect the error of the fitted model.
On one hand, \eqref{eq:asym_train_val_theorem1} is linear in $J$ so increasing $J$ can increase the error.
On the other hand, \eqref{eq:asym_train_val_oracle_risk} decreases for larger model spaces, so increasing $J$ may decrease the error.
We isolate these two behaviors via two simulation setups.

The data is generated as the sum of univariate functions
$Y = \sum_{j=1}^J g_j^*(X_j) + \sigma \epsilon$,
where $\epsilon$ are iid standard Gaussian random variables and $\sigma > 0$ is chosen such that the signal to noise ratio is two. $X$ is drawn from a uniform distribution over $\mathcal{X} = [-2, 2]^J$.
We fit models by minimizing \eqref{eq:smoothing_spline}.
To vary the number of free penalty parameters, we constrain certain $\lambda_j$ to be equal while allowing others to be completely free.
(For instance, for a single penalty parameter, we constrain $\lambda_j$ for $j=1,...,J$ to be the same value.) 
The penalty parameters are tuned using a training/validation split.

\noindent \textbf{Simulation 1}: The true function is the sum of identical sinusoids
$g_j^*(x_j) = \sin(x_j)$ for $j = 1,...,J$.
Since the univariate functions are the same, the oracle risk should be roughly constant as we increase the number of free penalty parameters.
The validation loss difference
\begin{align}
\left \| \sum_{j=1}^J \hat{g}^{(n_T)}_j(\hat{\boldsymbol{\lambda}}|T) - g^*_j \right \|_V^2 - 
\min_{\lambda \in \Lambda}
\left \| \sum_{j=1}^J \hat{g}^{(n_T)}_j(\boldsymbol{\lambda} | T) - g^*_j \right \|_V^2
\label{eq:excess_risk_sim}
\end{align}
should grow linearly in $J$ for this simulation setup.

\noindent \textbf{Simulation 2}: The true function is the sum of sinusoids with increasing frequency
$g_j^*(x_j) = \sin(x_j * 1.2^{j - 4})$for $j = 1,...,J$.
Since the Sobolev norms of $g_j^*$ increase with $j$, we expect that the penalty parameters that attain the oracle risk to be monotonically decreasing, e.g. ${\lambda}_1 > ... > {\lambda}_J$. 
As the number of penalty parameters increases, we expect the oracle risk to shrink.
If the oracle risk shrinks fast enough, performance of the selected model should improve.

For both simulations, we use $J = 8$.
Each simulation was replicated forty times with 200 training and 200 validation samples.
We consider $k = 1, 2, 4, 8$ free penalty parameters by structuring the penalty parameters in a nested fashion: for each $k$, we constrained $\{\lambda_{8\ell/k + j} \}_{j = 1,...,8/k}$ to be equal for $\ell= 0,...,k - 1$.
Penalty parameters were tuned using \texttt{nlm} in \texttt{R} with initializations at $\{\vec{1}, 0.1 \times \vec{1}, 0.01 \times \vec{1}\}$. 
We did not use grid-search since it is computationally intractable for large numbers of penalty parameters.
Multiple initializations were required since the validation loss is not convex in the penalty parameters.

As expected, the validation loss difference increases with the number of penalty parameters in Simulation 1 (Figure~\ref{fig:simulations}(a)).
To see if our oracle inequalities match the empirical results, we regressed the logarithm of the validation loss difference against the logarithm of the number of penalty parameters.
We fit the model using simulation results with at least two penalty parameters as the data is highly skewed for the single penalty parameter case.
We estimated a slope of 1.00 (standard error 0.15), which suggests that the validation loss difference grows linearly in the number of penalty parameters.
Interestingly, including the single parameter case gives us a slope of 1.45 (standard error 0.14).
This suggests that our oracle inequality might not be tight for the single penalty parameter case.

For Simulation 2, the validation loss of the selected model decreases as the number of penalty parameters increases.
As suggested in Figure~\ref{fig:simulations}(b), the validation loss of the selected model decreases because the oracle risk is decreasing at a faster rate than the rate at which the additional error \eqref{eq:asym_train_val_theorem1} grows.

These simulation results suggest that adding more hyper-parameters can improve model estimates.
Having a separate penalty parameter allows GAMs to fit components with differing smoothness.
However if we know a priori that the components have the same smoothness, then it is best to use a single penalty parameter.

\begin{figure}
	\centering
	\begin{subfigure}{0.6\textwidth}
		\includegraphics[width=\textwidth]{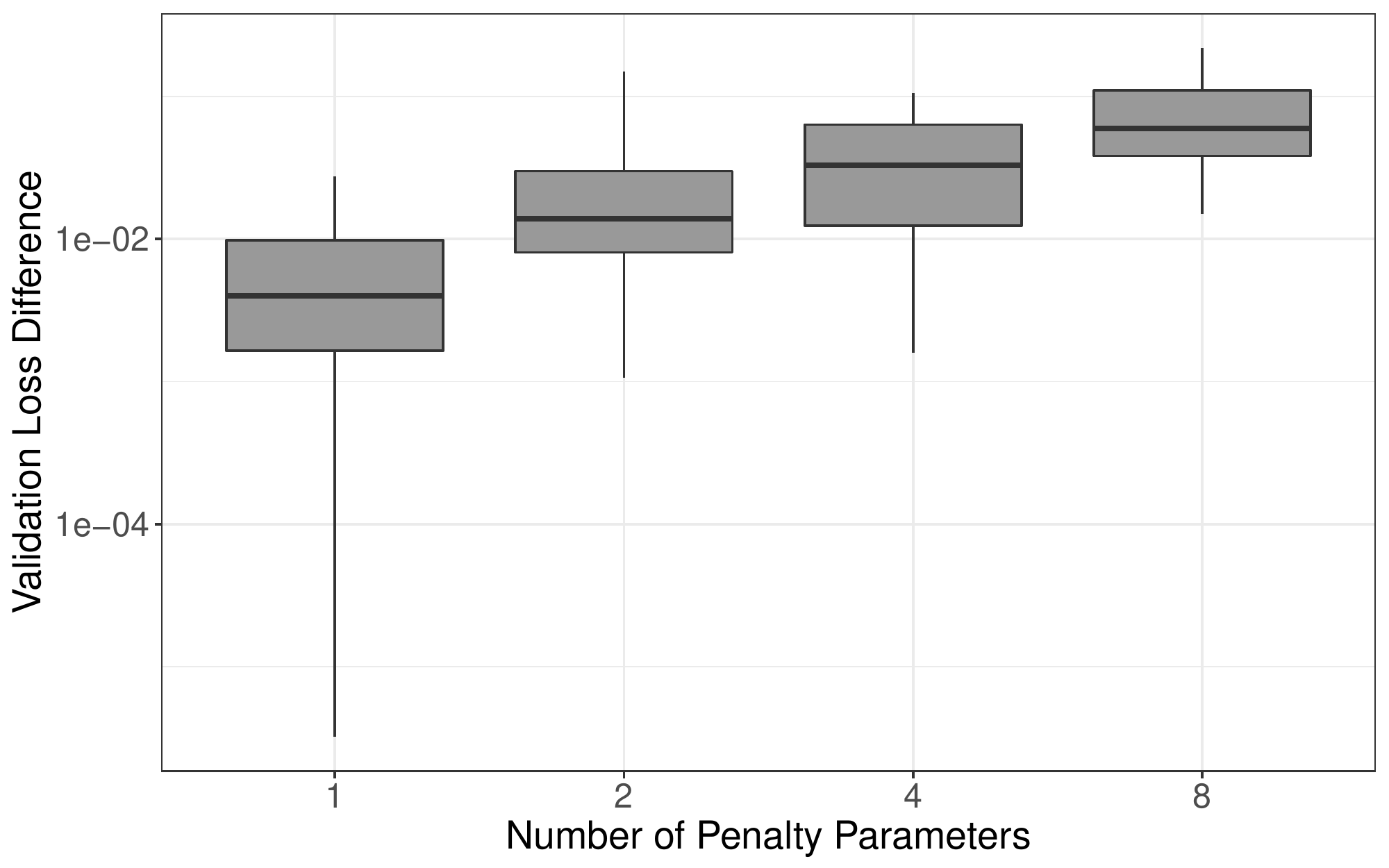}
		\caption{Simulation 1: the univariate additive components are the same}
	\end{subfigure}
	\begin{subfigure}{0.6\textwidth}
		\includegraphics[width=\textwidth]{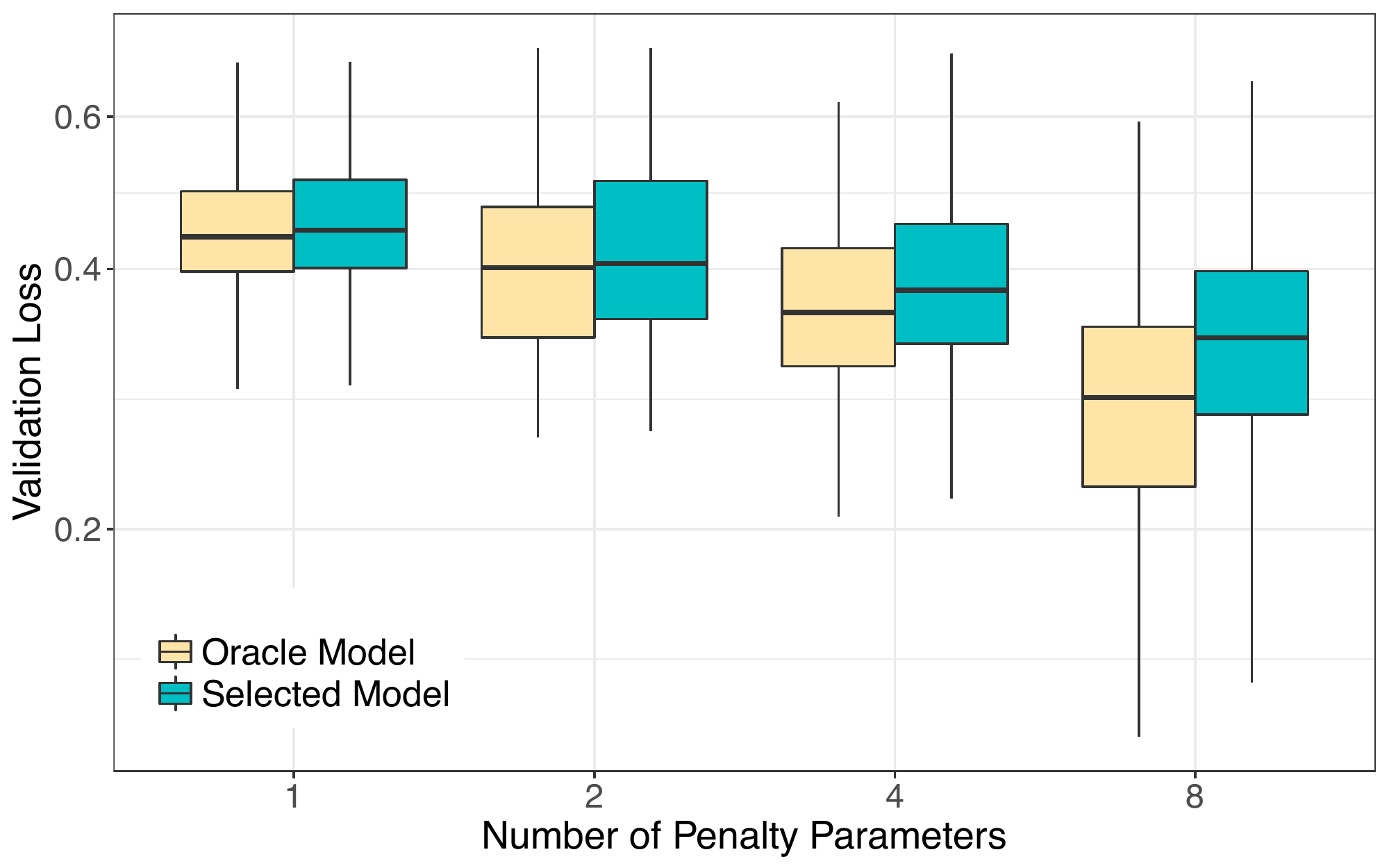}
		\caption{Simulation 2: the univariate additive components have differing levels of smoothness}
	\end{subfigure}
	\caption{
		Performance of generalized additive models as the number of free penalty parameters grows.
	}
	\label{fig:simulations}
\end{figure}

\section{Discussion}\label{sec:discussion}

In this manuscript, we have characterized the generalization error of split-sample procedures that tune multiple hyper-parameters. 
If the estimated models are Lipschitz in the hyper-parameters, the generalization error of the selected model is upper bounded by a combination of the oracle risk and a near-parametric term in the number of hyper-parameters.
These results show that adding hyper-parameters can decrease the generalization error of the selected model if the oracle risk decreases by a sufficient amount.
In the semi- or non-parametric setting, the error incurred from tuning hyper-parameters is dominated by the oracle risk asymptotically; adding hyper-parameters has a negligible effect on the generalization error of the selected model.
In the parametric setting, the error incurred from tuning hyper-parameters is on the same order as the oracle error; one should be careful about adding hyper-parameters, though they are not more ``costly'' than model parameters.

We also showed that many penalized regression examples satisfy the Lipschitz condition so our theoretical results apply.
This implies that fitting models with multiple penalties and penalty parameters can be desirable, rather than the usual case with one or two penalty parameters.

One drawback of our theoretical results is that we have assumed that selected hyper-parameter is a global minimizer of the validation loss.
Unfortunately this is not achievable in practice since the validation loss is not convex with respect to the hyper-parameters.
This problem is exacerbated when there are many hyper-parameters since it is computationally infeasible to perform an exhaustive grid-search. 
We hope to address this question in future research.

\appendix
\section{Supplementary Materials}
\label{sec:proofs}

We will use the following notation: for functions $f$ and $g$ and a dataset $D$ with $m$ samples, we denote the inner product of $f$ and $g$ at covariates $D$ as $\langle f,g \rangle_{D} = \frac{1}{m} \sum_{(x_i, y_i) \in D} f(x_i, y_i) g(x_i, y_i) $.

\subsection{A single training/validation split}
\label{appendix:train_val}

Theorem \ref{thrm:train_val} is a special case of Theorem \ref{thrm:train_val_complicated}, which applies to general model-estimation procedures. The proof is based on the so-called ``basic inequality'' below.

\begin{lemma}
	For any $\tilde{\boldsymbol{\lambda}} \in \tilde{\Lambda}$, we have
	\begin{equation}
	\label{thrm:basic_ineq}
	\left \| g^* - \hat{g}^{(n_T)}(\hat{\boldsymbol{\lambda}}|T) \right \|^2_V 
	- \left \| g^* - \hat{g}^{(n_T)}(\tilde{\boldsymbol{\lambda}}|T) \right \|^2_V
	\le 
	2 \left \langle \epsilon, \hat{g}^{(n_T)}(\tilde{\boldsymbol{\lambda}}|T) - \hat{g}^{(n_T)}(\hat{\boldsymbol{\lambda}}|T) \right \rangle_V
	\end{equation}
\end{lemma}

\begin{proof}
	The desired result can be attained by rearranging the definition of $\hat{\boldsymbol{\lambda}}$
	\begin{equation}
	\left \| y - \hat{g}^{(n_T)}(\hat{\boldsymbol{\lambda}}|T) \right \|^2_V \le
	\min_{\tilde{\boldsymbol{\lambda}} \in \tilde{\Lambda}} \left \| y - \hat{g}^{(n_T)}(\tilde{\boldsymbol{\lambda}}|T) \right \|^2_V.
	\end{equation}
\end{proof}

We are therefore interested in bounding the empirical process term in \eqref{thrm:basic_ineq}. A common approach is to use a measure of complexity of the function class. For a single training/validation split, where we treat the training set as fixed, we only need to consider the complexity of the fitted models from the model-selection procedure
\begin{equation}
\mathcal{G}(T)=\left\{ \hat{g}^{(n_T)}(\boldsymbol{\lambda}|T) : \boldsymbol{\lambda} \in \Lambda \right\}.
\end{equation}
This model class can be considerably less complex compared to the original function class $\mathcal{G}$, such as the special case in Theorem \ref{thrm:train_val} where we suppose $\mathcal{G}(T)$ is Lipschitz. For this proof, we will use metric entropy as a measure of model class complexity. We recall its definition below.
\begin{definition}
	Let $\mathcal{F}$ be a function class. Let the covering number $N(u, \mathcal{F}, \| \cdot \|)$ be the smallest set of $u$-covers of $\mathcal{F}$ with respect to the norm $\| \cdot \|$. The metric entropy of $\mathcal{F}$ is defined as the log of the covering number:
	\begin{equation}
	H (u, \mathcal{F}, \| \cdot \| ) = \log N(u, \mathcal{F}, \| \cdot \|).
	\end{equation}
\end{definition}

We will bound the empirical process term using the following Lemma, which is a simplification of Corollary 8.3 in \citet{van2000empirical}.

\begin{lemma}
	\label{lemma:cor83}
	Suppose $D^{(m)} = \{x_1,...,x_m\}$ are fixed and $\epsilon_1,...,\epsilon_m$ are independent random variables with mean zero and uniformly sub-gaussian with parameters $b$ and $B$. Suppose
	the model class $\mathcal{F}$ satisfies $\sup_{f\in\mathcal{F}}\|f\|_{D^{(m)}}\le R$
	and
	\[
	\int_{0}^{R}H^{1/2}(u,\mathcal{F},\|\cdot\|_{D^{(m)}})du \le \mathcal{J} (R).
	\]

	There is a constant $a > 0$ dependent only on $b$ and $B$ such that
	for all $\delta>0$ satisfying
	\[
	\sqrt{m}\delta\ge a(\mathcal{J} (R)\vee R),
	\]
	we have 
	\[
	Pr\left(\sup_{f\in\mathcal{F}}\left|\frac{1}{m}\sum_{i=1}^{m}\epsilon_{i}f(x_{i})\right|\ge\delta\right)
	\le 
	a\exp\left(-\frac{m\delta^{2}}{4a^{2}R^{2}}\right).
	\]
	
\end{lemma}

We are now ready to prove the oracle inequality. It uses a standard peeling argument.

\begin{theorem}
	\label{thrm:train_val_complicated}
	Consider a set of hyper-parameters $\Lambda$.
	Let training data $T$ be fixed, as well as the covariates of the validation set $X_V$.
	Let the oracle risk be denoted
	\begin{equation}
	\tilde{R}(X_V|T) = \argmin_{\lambda \in \Lambda} \left \| g^*-\hat{g}^{(n_T)}( \boldsymbol{\lambda} | T) \right \|_{V}^{2}.
	\end{equation}
	
	Suppose independent random variables $\epsilon_i$ for validation set $V$ have expectation zero and are uniformly sub-Gaussian with parameter $b$ and $B$.
	Suppose there is a function $\mathcal{J} (\cdot | T):\mathbb{R}\mapsto\mathbb{R}$ and constant $r > 0$ such that
	\begin{equation}
	\label{eq:dudley_bound}
	\int_{0}^{R}H^{1/2}(u,\mathcal{G}(T),\|\cdot\|_{V})du\le \mathcal{J} (R| T) \quad \forall R>r
	\end{equation}
	Also, suppose $\mathcal{J} \left(u | T \right)/u^{2}$ is non-increasing in $u$ for all $u > r$.
	
	Then there is a constant $c>0$ only depending on $b$ and $B$ such that for all $\delta$ satisfying
	\begin{equation}
	\label{eq:train_val_delta_condn}
	\sqrt{n_V}\delta^{2}
	\ge
	c \left ( 
	\mathcal{J}(\delta| T)
	\vee 
	\delta
	\vee
	\mathcal{J} \left (\tilde{R}(X_V|T)\middle | T
	\right ) 
	\vee
	4 \tilde{R}(X_V|T) \right ),
	\end{equation}
	we have
	\begin{align}
	Pr\left(
	\left\Vert g^* - \hat{g}^{(n_T)}( \hat{\boldsymbol{\lambda}} | T) \right\Vert _{V}^2 -
	\tilde{R}(X_V|T)
	\ge\delta^2
	\middle | 
	T, X_V
	\right )
	&\le c\exp\left(-\frac{n_{V}\delta^{4}}{
		c^{2}
		\tilde{R}(X_V|T)
	}\right) 
	+c\exp\left(-\frac{n_{V}\delta^{2}}{c^{2}}\right).
	\end{align}
\end{theorem}

\begin{proof}
	Consider any $\tilde{\boldsymbol{\lambda}} \in \tilde{\Lambda}$.
	We will use the simplified notation $\hat{g}(\hat{\boldsymbol{\lambda}}) \coloneqq \hat{g}^{(n_T)}(\hat{\boldsymbol{\lambda}} | T)$ and $\hat{g}(\tilde{\boldsymbol{\lambda}}) \coloneqq \hat{g}^{(n_T)}(\tilde{\boldsymbol{\lambda}} | T)$. In addition, the following probabilities are all conditional on $X_V$ and $T$ but we leave them out for readability.
	\begin{align}
	& \Pr\left(
	\left\Vert \hat{g}(\hat{\boldsymbol{\lambda}})-g^{*}\right\Vert _{V}^{2}
	- \tilde{R}(X_V|T)
	\ge \delta^2
	\right) \label{eq:train_val_prob}\\
	& = \sum_{s=0}^{\infty}
	\Pr\left(
	2^{2s}\delta^{2}
	\le \left\Vert \hat{g}(\hat{\boldsymbol{\lambda}})-g^{*}\right\Vert _{V}^{2}
	-\tilde{R}(X_V|T)
	\le 2^{2s+2}\delta^{2}\right) 
	\label{eq:peeled} \\
	&\le \sum_{s=0}^{\infty}
	\Pr\left(
	2^{2s}\delta^{2}
	\le 2\left\langle \epsilon,\hat{g}(\hat{\boldsymbol{\lambda}})-\hat{g}(\tilde{\boldsymbol{\lambda}})\right\rangle _{V}\right. \label{eq:peel_ineq}\\
	& \qquad  \left.\wedge \left\Vert \hat{g}(\hat{\boldsymbol{\lambda}})-\hat{g}(\tilde{\boldsymbol{\lambda}})\right\Vert_{V}^{2}\le2^{2s+2}\delta^{2}+ 2\left|\left\langle \hat{g}(\tilde{\boldsymbol{\lambda}})-\hat{g}(\hat{\boldsymbol{\lambda}}),\hat{g}(\tilde{\boldsymbol{\lambda}})-g^{*}\right\rangle _{V}\right| \right ),
	\end{align}
	where we applied the basic inequality \eqref{thrm:basic_ineq} in the last line.
	Each summand in \eqref{eq:peel_ineq} can be bounded by splitting the event into the cases where either $2^{2s+2} \delta^2$ or $2\left|\left\langle \hat{g}(\tilde{\boldsymbol{\lambda}})-\hat{g}(\hat{\boldsymbol{\lambda}}),\hat{g}(\tilde{\boldsymbol{\lambda}})-g^{*}\right\rangle _{V}\right|$ is larger. Splitting up the probability and applying Cauchy Schwarz gives us the following bound for \eqref{eq:train_val_prob}
	\begin{align}
	& Pr\left(
	\sup_{\boldsymbol{\lambda} \in \Lambda: \left\Vert \hat{g}({\boldsymbol{\lambda}})-\hat{g}(\tilde{\boldsymbol{\lambda}})\right\Vert _{V}
		\le
		4\left\Vert \hat{g}(\tilde{\boldsymbol{\lambda}})-g^{*}\right\Vert _{V}}
	2\left\langle \epsilon,\hat{g}({\boldsymbol{\lambda}})-\hat{g}(\tilde{\boldsymbol{\lambda}})\right\rangle _{V}
	\ge 
	\delta^{2}
	\right)
	\label{eq:train_val_1}
	\\
	& + \sum_{s=0}^{\infty} Pr\left(
	\sup_{\boldsymbol{\lambda} \in \Lambda: \left\Vert \hat{g}({\boldsymbol{\lambda}})-\hat{g}(\tilde{\boldsymbol{\lambda}})\right\Vert _{V}
		\le
		2^{s+3/2}\delta}
	2\left\langle \epsilon,\hat{g}({\boldsymbol{\lambda}})-\hat{g}(\tilde{\boldsymbol{\lambda}})\right\rangle _{V}
	\ge
	2^{2s} \delta^{2}
	\right)
	\label{eq:train_val_2}.
	\end{align}
	
	We can bound both \eqref{eq:train_val_1} and \eqref{eq:train_val_2} using Lemma \ref{lemma:cor83}. For our choice of $\delta$ in \eqref{eq:train_val_delta_condn},
	there is some constant $a>0$ dependent only on $b$ such that \eqref{eq:train_val_1} is bounded above by
	\[ 
	a\exp\left(-\frac{n_{V}\delta^{4}}{4a^{2}\left(16\left\Vert \hat{g}(\tilde{\boldsymbol{\lambda}})-g^{*}\right\Vert _{V}^{2}\right)}\right).
	\]
	In addition, our choice of $\delta$ from \eqref{eq:train_val_delta_condn} and our assumption that $\psi(u)/u^2$ is non-increasing implies that the condition in Lemma \ref{lemma:cor83} is satisfied for all $s=0,1,...,\infty$ simultaneously. Hence for all $s=0,1,...,\infty$, we have
	\begin{align}
	Pr\left(
	\sup_{\boldsymbol{\lambda} \in \Lambda: \left\Vert \hat{g}({\boldsymbol{\lambda}})-\hat{g}(\tilde{\boldsymbol{\lambda}})\right\Vert _{V}
		\le
		2^{s+3/2}\delta}
	2\left\langle \epsilon,\hat{g}({\boldsymbol{\lambda}})-\hat{g}(\tilde{\boldsymbol{\lambda}})\right\rangle _{V}
	\ge
	2^{2s} \delta^{2}
	\right)
	& \le 
	a\exp\left(-n_{V}\frac{2^{4s-2}\delta^{4}}{4a^{2}2^{2s+3}\delta^{2}}\right).
	\end{align}
	
	Putting this all together, we have that there is a constant $c$ such that \eqref{eq:train_val_prob} is bounded above by
	\begin{equation}
	c\exp\left(-\frac{n_{V}\delta^{4}}{c^{2} \tilde{R}(X_V|T)}\right)
	+
	c\exp\left(-\frac{n_{V} \delta^2}{c^{2}}\right).
	\end{equation}
	
\end{proof}

We can apply Theorem \ref{thrm:train_val_complicated} to get Theorem \ref{thrm:train_val}. Before proceeding, we determine the entropy of $\mathcal{G}(T)$ when the functions are Lipschitz in the hyper-parameters.

\begin{lemma}
	\label{lemma:covering_cube}
	Let $\Lambda = [\lambda_{\min}, \lambda_{\max}]^J$ where $\lambda_{\min} \le \lambda_{\max}$. Suppose $\mathcal{G}(T)$ is Lipschitz with function $C(\cdot | T)$ over $\boldsymbol{\lambda}$.
	Then the entropy of $\mathcal{G}(T)$ with respect to $\| \cdot \|$ is
	\begin{equation}
	H\left(u, \mathcal{G}(T),\|\cdot\|\right) \le
	J \log \left(\frac{4 \|C(\cdot | T)\| \left(\lambda_{max}-\lambda_{min}\right)+2u}{u}\right).
	\end{equation}
\end{lemma}
\begin{proof}
	Using a slight variation of the proof for Lemma 2.5 in \citet{van2000empirical}, we can show
	\begin{align}
	N\left(u,\Lambda,\|\cdot\|_{2}\right) \le \left(\frac{4\left(\lambda_{max}-\lambda_{min}\right)+2u}{u}\right)^{J}.
	\end{align}
	Under the Lipschitz assumption, a $\delta$-cover for $\Lambda$
	is a $\|C(\cdot | T)\|\delta$-cover for $\mathcal{G}(T)$. The covering number for $\mathcal{G}(T)$ wrt $\|\cdot\|$ is bounded by the covering number for $\Lambda$ as follows
	\begin{eqnarray}
	N\left(u,\mathcal{G}(T),\|\cdot\|\right)
	&\le& N\left(\frac{u}{\|C(\cdot | T)\|},\Lambda,\|\cdot\|_{2}\right)\\
	&\le& \left(\frac{4\left(\lambda_{max}-\lambda_{min}\right)+2u/\|C(\cdot | T)\|}{u/\|C(\cdot | T)\|}\right)^{J}.
	\end{eqnarray}
\end{proof}

\subsubsection{Proof for Theorem \ref{thrm:train_val}}
\begin{proof}
	By Lemma \ref{lemma:covering_cube}, we have
	\begin{align}
	\int_{0}^{R}H^{1/2}(u,\mathcal{G}(T),\|\cdot\|_{V})du 
	&= \int_{0}^{R} \left ( 
	J \log \left(\frac{4 \|C_\Lambda\|_V \Delta_{\Lambda}+2u}{u}\right)
	\right )^{1/2}
	du\\
	& \le J^{1/2}\int_{0}^{R}\left[
	\log\left(
	\frac{4 \|C_\Lambda(\cdot | T)\|_V \Delta_{\Lambda} + 2R }
	{u}
	\right)
	\right]^{1/2}du\\
	& = J^{1/2}R \int_{0}^{1}\left[
	\log\left(
	\frac{4 \|C_\Lambda(\cdot | T)\|_V \Delta_{\Lambda} + 2R }
	{vR}
	\right)
	\right]^{1/2}dv\\
	& \le J^{1/2}R \int_{0}^{1}
	\log^{1/2}\left(
	\frac{4 \|C_\Lambda(\cdot | T)\|_V \Delta_{\Lambda} + 2R}
	{R}
	\right)
	+
	\log^{1/2}(1/v)
	dv\\
	& < J^{1/2}R \left (
	\log^{1/2}\left(
	\frac{4 \|C_\Lambda(\cdot | T)\|_V \Delta_{\Lambda} + 2R}
	{R}
	\right)
	+
	1
	\right ).
	\end{align}
	If we restrict $R > n^{-1}$, then for an absolute constant $c$, we have
	\begin{equation}
	\label{eq:train_val_entropy}
	\int_{0}^{R}H^{1/2}(u,\mathcal{G}(T),\|\cdot\|_{V})du
	\le
	\mathcal{J}(R) 
	\coloneqq c R\left ( J \log(\|C_\Lambda(\cdot |T)\|_V \Delta_{\Lambda} n + 1) \right )^{1/2}.
	\end{equation}
	Applying Theorem \ref{thrm:train_val_complicated}, we get our desired result.
\end{proof}

\subsection{Cross-validation}
\label{app:cv}
In order to obtain an oracle inequality for averaged version of cross-validation, we need to extend Theorem 3.5 in \citet{lecue2012oracle}.
Let the class of fitted functions for given training data $T$ be denoted
$$
\mathcal{G}(T) = \{\hat{g}^{(n_T)}(\boldsymbol{\lambda}| T) : \boldsymbol{\lambda} \in \Lambda\}.
$$
In \citet{lecue2012oracle}, they assume that there is a function $\mathcal{J}$ that uniformly bounds the size of the class $\mathcal{G}(T)$ for any training data $T$.
However the complexity of $\mathcal{G}(T)$ depends on training data -- for instance, if there is a lot of noise in the training data, the size of $\mathcal{G}(T)$ can be very high.
In our extension, we allow the function $\mathcal{J}$ to depend on the training data.

Throughout this section, we use Talagrand's gamma function \citep{talagrand2006generic} to characterize the size of a function class.
We present it below as it will be used later on.
\begin{definition}
	For metric space $(T,d)$ and $\alpha \ge 0$, define
	$$
	\gamma_\alpha(T,d) = \inf \sup_{t\in T} \sum_{s = 0}^{\infty} 2^{s/\alpha}d(t, T_s)
	$$
	where the infimum is taken over all sequences $\{T_s: s\in \mathbb{N}, T_s \subseteq T, |T_s| \le 2^{2^s} \}$.
	(Here, $|A|$ denotes the cardinality of the set $A$.)
\end{definition}

We begin with some notation.
Suppose we have a measurable space $(\mathcal{Z}, \mathcal{T})$ where we observe $Z = (X,y)$ random variables with values in $\mathcal{Z}$.
Let $\mathcal{G}$ is a class of measurable functions from $\mathcal{Z} \mapsto \mathbb{R}$; the model-estimation procedure selects functions from the class $\mathcal{G}$.
In contrast to the main manuscript, we will consider a very general setting.
In particular, the noise $\epsilon = y - E[y | X=x]$ is not necessarily independent of $X$.
In addition, we consider a general loss function $Q: \mathcal{Z} \times \mathcal{G} \mapsto \mathbb{R}$ (rather than solely the least squares loss).
Define the risk function $R(g)$ as the expected loss $\mathbb{E} Q(Z, g)$ and suppose the risk function is convex.
Let $\bar{g}^{(n)}(D^{(n)})$ denote the averaged version of cross-validation and $g^*$ denote the minimizer of the risk function over $\mathcal{G}$.

In this more general setting, we require a more general version of Assumption~\ref{assump:tail_margin}:
\begin{assump}
	\label{assump:tail_margin_general}
	There exist constants $K_0, K_1 \ge 0$ and $\kappa \ge 1$ such that for any $m \in \mathbb{N}$ and any dataset $D^{(m)}$,
	\begin{align}
	\left \| Q(\cdot, \hat{g}^{(n_T)}(\boldsymbol{\lambda} | D^{(n_T)}) - Q(\cdot, g^*) \right \|_{L_{\psi_1}} & \le K_0
	\label{eq:cv_assump1_app}\\
	\left \| Q(\cdot, \hat{g}^{(n_T)}(\boldsymbol{\lambda} | D^{(n_T)}) - Q(\cdot, g^*)  \right \|_{L_2}
	& \le 
	K_1 \left ( R(\hat{g}^{(n_T)}(\boldsymbol{\lambda}|D^{(n_T)})) - R(g^*) \right )^{1/2\kappa}.
	\label{eq:cv_assump2_app}
	\end{align}
\end{assump}

Our theorem relies on the basic inequality established in Lemma 3.1 in \citet{lecue2012oracle}.
We reproduce it here for convenience.
From henceforth, $c_i > 0$ denotes absolute constants, that may not necessarily be the same if they share the same subscript.
\begin{lemma}
	For any constant $a >0$, we have the following inequality
	\begin{align}
	\begin{split}
	\mathbb{E}_{D^{(n)}} \left(
	R(\bar g^{(n)}(D^{(n)})) - R(g^*) 
	\right)
	& \le
	(1 + a)
	\inf_{\lambda \in \Lambda}
	\left[
	\mathbb{E}_{D^{(n_V)}}
	R(\hat{g}^{(n_V)}(\boldsymbol{\lambda}| D^{(n_V)})) - R(g^*)
	\right]\\
	& \quad +
	\mathbb{E}_{D^{(n)}}
	\sup_{\lambda\in \Lambda}
	\left[
	(P - (1 + a)P_{n_V})
	\left(
	Q(\cdot, \hat{g}^{(n_T)}(\boldsymbol{\lambda} | D^{(n_T)}))
	- Q(\cdot, g^*)
	\right )
	\right]
	\end{split}
	\label{eq:basic_ineq_cv}
	\end{align}
	where $P_{n_V} = 1/n_V \sum_{i=n_T +1}^n \delta_{Z_i}$ is the empirical probability measure on $\{Z_{n_T + 1}, ..., Z_n\}$.
	\label{lemma:lecue_basic_ineq}
\end{lemma}

We need to bound the supremum of the second term on the right hand side, which is a shifted empirical process term.
Lemma 3.4 in \citet{lecue2012oracle} already bounds the shifted empirical process term.
However to extend their result to our purposes, we restate it to clarify the conditional dependencies.
This allows us to introduce two new functions $h$ and $J_\delta$ that will be used later on.
\begin{lemma}
	Let $\mathcal{Q}(D^{(m)})\equiv\left\{ Q(\lambda|D^{(m)}):\lambda\in\Lambda\right\} $
	and $\mathcal{Q}\equiv\cup_{m\in\mathbb{N}}\cup_{D^{(m)}}\mathcal{Q}(D^{(m)})$.
	Suppose there exists $C_{1}>0$ and an increasing function $G(\cdot)$
	such that $\forall Q\in\mathcal{Q}$, 
	\[
	\|Q(Z)\|_{L_{2}}\le G\left(\mathbb{E}Q(Z)\right).
	\]
	Let $n_{T},n_{V}\in\mathbb{N}$.
	Suppose there exists a function $h$ that maps training data $D^{(n_T)}$ to $\mathbb{R}^+$,
	a function $J_\delta :\mathbb{R}^+ \mapsto \mathbb{R}^+$ indexed by $\delta > 0$,
	and a constant $w_{\min}>0$ such that for any dataset $D^{(n_{T})}$ and any $w \ge w_{\min}$,
	\begin{align}
	h(D^{(n_{T})})\le\delta\implies\frac{\log n_{V}}{\sqrt{n_{V}}}\gamma_{1}\left(\mathcal{Q}_{w}^{L_{2}}(D^{(n_{T})}),\|\cdot\|_{L_{\psi_{1}}}\right)+\gamma_{2}\left(\mathcal{Q}_{w}^{L_{2}}(D^{(n_{T})}),\|\cdot\|_{L_{2}}\right)\le J_{\delta}(w)
	\end{align}
	where $\mathcal{Q}_{w}^{L_{2}}(D^{(n_{T})})\equiv\left\{ Q\in\mathcal{Q}(D^{(n_{T})}):\|Q(Z)\|_{L_{2}}\le G(w)\right\}$.
	
	Then there exists absolute constants $L,c>0$ such that for all
	$w\ge w_{\min}$ and all $u\ge1$,
	\begin{align}
	\Pr\left(
	\sup_{Q\in\mathcal{Q}(D^{(n_{T})}): PQ \le w}
	\left(\left({P}-P_{n_{V}}\right)Q\right)_{+}
	\le uL\frac{J_{\delta}(w)}{\sqrt{n_{V}}}
	\middle | h\left(D^{(n_{T})}\right)
	\le \delta
	\right)
	\ge
	1-L\exp(-cu).
	\end{align}
	\label{lemma:lecue_prelim}
\end{lemma}

Now that we have established a concentration inequality for the function class $\{Q\in\mathcal{Q}(D^{(n_{T})}): PQ \le w\}$, we need to aggregate the results to establish a concentration inequality for the function class $\mathcal{Q}(D^{(n_{T})})$.
Again, we use Lemma 3.2 in \citet{lecue2012oracle} but restate it using our new functions $h$ and $J_\delta$.
\begin{lemma}
	Let $a>0$. Let $\mathcal{Q}(D^{(m)})\equiv\left\{ Q(\lambda|D^{(m)}):\lambda\in\Lambda\right\} $
	be a set of measurable functions.
	For all $m\in\mathbb{N}$ and any dataset $D^{(m)}$, suppose $\mathbb{E}Q(Z)\ge 0 $ for all $Q\in\mathcal{Q}\left(D^{(m)}\right)$.
	
	Suppose for any $n_{T},n_{V}\in\mathbb{N}$ and dataset $D^{(n_{T})}$
	there exists some absolute constant $L,c>0$ such that for all $w\ge w_{\min}$
	and for all $u\ge1$,
	\[
	\Pr\left(
	\sup_{Q\in\mathcal{Q}(D^{(n_{T})}): PQ \le w}
	\left(\left({P}-P_{n_{V}}\right)Q\right)_{+}\le uL\frac{J_{\delta}(w)}{\sqrt{n_{V}}}
	\middle |
	h\left(D^{(n_{T})}\right)\le\delta\right)
	\ge 1-L\exp(-cu).
	\]
	For any $\delta > 0$, suppose $J_{\delta}$ is strictly increasing and its inverse is strictly convex.
	Let $\psi_{\delta}$ be the convex conjugate of $J_{\delta}^{-1}$,
	e.g. $\psi_{\delta}(u)=\sup_{v>0}uv-J_{\delta}^{-1}(v)$ for all $u>0$.
	Assume there is a $r\ge1$ such that $x>0\mapsto\psi_\delta(x)/x^{r}$ decreases.
	For all $q>1$ and $u\ge1$, define
	\[
	\tilde \psi_{q,\delta}(u)=\psi_{\delta}\left(\frac{2q^{r+1}(1+a)u}{a\sqrt{n_{V}}}\right)\vee w_{\min}.
	\]

	Then there exists a constant $L_{1}$ that only depends on $L$ such
	that for every $u\ge1$,
	\[
	\Pr\left(
	\sup_{Q\in\mathcal{Q}(D^{(n_{T})})}
	\left(\left({P}-(1+a)P_{n_{V}}\right)Q\right)_{+}
	\le \frac{a \tilde{\psi}_{q,\delta}(u/q)}{q}
	\middle | h\left(D^{(n_{T})}\right)\le\delta
	\right)
	\ge 1-L_{1}\exp(-cu).
	\]

	Moreover, assume that $\psi_{\delta}(x)$ is an increasing function in $x$ such that $\psi_{\delta}(\infty)=\infty$.
	Then there exists a constant $c_{1}$ that depends only on $L$ and $c$ such that
	\begin{align}
	\mathbb{E}\left[\sup_{Q\in\mathcal{Q}(D^{(n_{T})})}
	\left(\left(P-(1+a)P_{n_{V}}\right)Q\right)_{+}
	\middle |
	h\left(D^{(n_{T})}\right)\le\delta
	\right]
	\le\frac{ac_{1} \tilde{\psi}_{q,\delta}(1/q)}{q}.
	\label{eq:lecue_exp_cond}
	\end{align}
	\label{lemma:lecue_prelim2}
\end{lemma}

Finally, we are ready to bound the expectation of the shifted empirical process term in \eqref{eq:basic_ineq_cv}.
We accomplish this via a simple chaining argument; we omit its proof as this is a standard application of the chaining argument.
\begin{lemma}
	\label{lemma:chain}
	Consider any $a>0$.
	Suppose there exists a constant $c_{1}$ such that for any $n_{T},n_{V} \in \mathbb{N}$, $\delta>0$,
	and $q>1$, \eqref{eq:lecue_exp_cond} holds.
	Then for any $\sigma > 0$, we have
	\[
	\mathbb{E}\left[\sup_{Q\in\mathcal{Q}(D^{(n_{T})})}\left(\left({P}-(1+a)P_{n_{V}}\right)Q\right)_{+}\right]
	\le
	\frac{ac_1}{q}
	\left(
	\tilde{\psi}_{q,2\sigma}(1/q)
	+\sum_{k=1}^{\infty}\Pr\left(h\left(D^{(n_{T})}\right)\ge2^{k}\sigma\right)
	\tilde{\psi}_{q,2^{k}\sigma}(1/q)
	\right)
	.
	\]
\end{lemma}
Putting Lemmas \ref{lemma:lecue_basic_ineq} and \ref{lemma:chain} together, we have the following result.
\begin{theorem}
	\label{thrm:jean_cv}
	Consider a set of hyper-parameters $\Lambda$. Consider a loss function $Q:(\mathcal{Z}, \mathcal{G}) \mapsto \mathbb{R}$ with convex risk function $R: \mathcal{G} \mapsto \mathbb{R}$. Let
	$$
	\mathcal{Q} = \{ 
	Q(\cdot, \hat{g}^{(n_T)}(\boldsymbol{\lambda} | D^{(n_T)}) - Q(\cdot, g^*) : \boldsymbol{\lambda} \in \Lambda \}.
	$$
	Suppose Assumption~\ref{assump:tail_margin_general} holds.
	Suppose there is an $w_{\min} > 0$ and
	functions $h: {\mathcal{Z}}^{(n_T)} \mapsto \mathbb{R}$
	and $\mathcal{J}_\delta: \mathbb{R}\mapsto \mathbb{R}$ such that
	for all $w \ge w_{\min}$,
	\begin{align}
	\label{eq:h_to_J}
	h(D^{(n_{T})})\le\delta\implies
	\frac{\log n_{V}}{\sqrt{n_{V}}}
	\gamma_{1}\left(\mathcal{Q}_{w}^{L_{2}}(D^{(n_{T})}),\|\cdot\|_{L_{\psi_{1}}}\right)
	+\gamma_{2}\left(\mathcal{Q}_{w}^{L_{2}}(D^{(n_{T})}),\|\cdot\|_{L_{2}}\right)\le \mathcal{J}_{\delta}(w)
	\end{align}
	where $\mathcal{Q}_w = \{Q \in \mathcal{Q}: \| Q \|_{L_2} \le w^{1/2\kappa} \}$.
	Moreover, suppose that for all $\delta > 0$, $J_\delta$ is a strictly increasing function and $\mathcal{J}_\delta^{-1}(\epsilon)$ is strictly convex.
	Let the convex conjugate of $\mathcal{J}^{-1}_\delta$ be denoted $\psi_\delta$.
	Suppose $\psi_\delta(x)$ increases in $x$, $\psi_\delta(\infty ) = \infty$, and there exists $r \ge 1$ such that $\psi_\delta(x)/x^r$ decreases.
	
	Consider any $\sigma > 0$. Then there is a constant $c > 0$ such that for every $a > 0$ and $q > 1$, the following inequality holds
	\begin{align}
	\begin{split}
	\mathbb{E}_{D^{(n)}} 
	\left(
	R\left(\bar{g} ( {D^{(n)}} ) \right )
	- R (g^*)
	\right)
	&\le
	(1+a) \inf_{\boldsymbol{\lambda} \in \Lambda} 
	\mathbb{E}_{D^{(n_T)}}
	\left(
	R\left(\bar{g} ( \hat{\boldsymbol \lambda} | {D^{(n)}} ) \right )
	- R (g^*)
	\right)
	\\
	& +
	\frac{ac}{q}
	\left(
	\tilde{\psi}_{q,2\sigma}(1/q)
	+\sum_{k=1}^{\infty}
	\Pr\left(h\left(D^{(n_{T})}\right)\ge2^{k}\sigma\right)
	\tilde{\psi}_{q,2^{k}\sigma}(1/q)
	\right).
	\label{eq:cv_oracle_ineq}
	\end{split}
	\end{align}
	where $\tilde{\psi}_{q, \delta}(u) = \psi_\delta\left(\frac{2q^{r+1}(1 + a)u}{a\sqrt{n_V}}\right) \vee w_{\min}$ for all $u > 0$.
\end{theorem}
Of course, this theorem is only useful if we can show that $h(D^{(n_T)})$ is bounded with high probability.
For instance, in an example in the main manuscript, we show that $h(D^{(n_T)})$ has sub-exponential tails; so the latter term in \eqref{eq:cv_oracle_ineq} is well-controlled.

We now apply Theorem~\ref{thrm:jean_cv} to prove Theorem~\ref{thrm:kfold}.
Recall that Theorem~\ref{thrm:kfold} concerns the squared error loss $Q((x,y), g) = (y - g(x))^2$ and only considers model-estimation methods where the estimated functions are Lipschitz in the hyper-parameters.
First we need the following lemma that describes the relationship between Lipschitz functions
\begin{lemma}
	Suppose the same conditions as Theorem~\ref{thrm:jean_cv}.
	Suppose Assumptions~\ref{assump:lipschitz} and \ref{assump:tail_margin} hold.
	Also suppose that $\|\epsilon\|_{L_{\psi_2}} = b <\infty$.
	Define 
	$\mathcal{Q}_{w}^{L_{2}} = \{g^* - \hat{g}(\boldsymbol{\lambda}|D^{(n_{T})}) : P (g^* - \hat{g}(\boldsymbol{\lambda}|D^{(n_{T})}))^2 < w\}$
	for $w > 0$.
	Then there is an absolute constant $c_0 > 0$ such that
	\begin{align}
	N\left(\mathcal{Q}_{w}^{L_{2}}(D^{(n_{T})}),u,\|\cdot\|_{L_{2}}\right)\le N\left(\Lambda,\frac{u}{c_0 \left(b +\sqrt{w}\right)
		\|C_\Lambda(x|D^{(n_{T})})\|_{L_{2}}},\|\cdot\|_{2}\right).
	\end{align}
	then we also have
	\begin{align}
	N\left(\mathcal{Q}_{w}^{L_{2}}(D^{(n_{T})}),u,\|\cdot\|_{L_{\psi_{1}}}\right)
	\le N\left(
	\Lambda,
	\frac{u}{c_{K_0, b}\|C_\Lambda(x|D^{(n_{T})})\|_{L_{\psi_{2}}}},\|\cdot\|_{2}\right)
	\end{align}
	for a constant $c_{K_0, b} > 0$ that only depends on $K_0$ and $b$.
	\label{lemma:covering_lipschitz}
\end{lemma}
\begin{proof}
	Let us first consider a general norm $\|\cdot \|$ such that for any random variables $X, Y$, we have $\|XY\| \le \|X\|_* \|Y\|_*$.
	Then for all $\boldsymbol{\lambda} \in \Lambda$ such that
	$P (g^* - \hat{g}(\boldsymbol{\lambda} | D^{n_T}))^2 \le w$, we have
	\begin{align}
	& \left \|
	Q(\cdot , \hat{g}^{(n_T)}(\boldsymbol{\lambda}^{(1)}|D^{(n_T)})(x))
	- Q(\cdot , \hat{g}^{(n_T)}(\boldsymbol{\lambda}^{(2)}|D^{(n_T)})(x))
	\right \|\\
	& = \left \|
	\left (y - \hat{g}^{(n_T)}(\boldsymbol{\lambda}^{(1)}|D^{(n_T)})(x) \right) ^2
	- \left (y - \hat{g}^{(n_T)}(\boldsymbol{\lambda}^{(2)}|D^{(n_T)})(x) \right) ^2
	\right \|
	\label{eq:loss_diff}
	\\
	&\left(\hat{g}^{(n_T)}(\boldsymbol{\lambda}^{(2)}|D^{(n_T)})(x) - \hat{g}^{(n_T)}(\boldsymbol{\lambda}^{(1)}|D^{(n_T)})(x)\right )^2 \|\\
	\begin{split}
	& \le
	\left \|2\epsilon + g^*(x) - \hat{g}(\boldsymbol{\lambda}^{(1)} | D^{(n_T)})(x)
	+ g^*(x) - \hat{g}(\boldsymbol{\lambda}^{(2)} | D^{(n_T)})(x) \right \|_* \\
	& \quad \times \left \| \hat{g}^{(n_T)}(\boldsymbol{\lambda}^{(2)}|D^{(n_T)})(x) - \hat{g}^{(n_T)}(\boldsymbol{\lambda}^{(1)}|D^{(n_T)})(x) \right \|_* \\
	\end{split}\\
	& \le  \left (2 \|\epsilon\|_* +
	2 \sup_{\lambda \in \Lambda: P(g^* - \hat{g}(\boldsymbol{\lambda} | D^{n_T}))^2 \le w} 
	\left \| g^*(x) - \hat{g}(\boldsymbol{\lambda}^{(1)} | D^{(n_T)})(x) \right \|_* 
	\right)
	\left \|C_\Lambda (x | D^{(n_T)}) \right \|_*
	\|\boldsymbol{\lambda}^{(2)} - \boldsymbol{\lambda}^{(1)} \|_2
	\label{eq:lipschitz_connect}
	\end{align}
	
	For $\|\cdot \| = \|\cdot\|_{L_{2}}$, the $L_2$ norm is its own dual norm so \eqref{eq:lipschitz_connect} reduces to
	$$
	c_0 \left(b +\sqrt{w}\right)
	\|C_\Lambda (x | D^{(n_T)})\|_{L_{2}}\|\boldsymbol{\lambda}^{(1)}-\boldsymbol{\lambda}^{(2)}\|_{2}$$
	for an absolute constant $c_0 > 0$.
	
	For $\|\cdot \| = \|\cdot\|_{L_{\psi_{1}}}$, the dual of the $L_{\psi_1}$ norm is $L_{\psi_2}$.
	Thus applying Assumption~\ref{assump:tail_margin} and the fact that $\|\epsilon\|_{L_{\psi_2}} = b < \infty$, \eqref{eq:lipschitz_connect} reduces to 
	$$
	2\left(
	b
	+ K_0
	\right)
	\|C_\Lambda (x | D^{(n_T)})\|_{L_{\psi_{2}}}
	\|\boldsymbol{\lambda}^{(1)}-\boldsymbol{\lambda}^{(2)}\|_{2}.
	$$
	
\end{proof}

Talagrand's gamma function of a class $T$ can be bounded by Dudley's integral
\begin{align}
\gamma_{\alpha}(T,D)\le c\int_{0}^{\text{Diam}(T,d)}\left(\log N(T,\epsilon,d)\right)^{1/\alpha}d\epsilon
\label{eq:bound_gamma}
\end{align}
\citep{talagrand2006generic}.
Combining the above bound with Lemma~\ref{lemma:covering_lipschitz} gives the following lemma.

\begin{lemma}
	Suppose Assumptions~\ref{assump:lipschitz} and \ref{assump:tail_margin} hold.
	Suppose $\|\epsilon\|_{L_{\psi_2}} = b < \infty$.
	Define $\mathcal{Q}_w^{L_2}$ as before.
	For $\Lambda$, let $\Delta_{\Lambda}=(\lambda_{\max}-\lambda_{\min}) \vee 1$.
	Let $w>0$.
	Let $\mathcal{Q}_{w}^{L_{2}}(D^{(n_{T})})$ be defined as before.
	
	Then there exist absolute constants $c_0, c_1 >0$ and a constant $c_{K_0, b} > 0$ such that
	\begin{align}
	\gamma_{2}\left(\mathcal{Q}_{w}^{L_{2}}(D^{(n_{T})}),\|\cdot\|_{L_{2}}\right)
	& \le	c_0 \sqrt{w J}
	\left[\sqrt{
		\log\left(
		\left (\frac{b}{\sqrt{w}} + 1 \right )
		\Delta_{\Lambda}\|C_\Lambda(x|D^{(n_{T})})\|_{L_{2}} + 1
		\right)
	}
	+1\right]\\
	\gamma_{1}\left(\mathcal{Q}_{w}^{L_{2}}(D^{(n_{T})}),\|\cdot\|_{L_{\psi_{1}}}\right)
	& \le c_{1}JK_0\left[
	\log\left(
	\Delta_{\Lambda}\|C_\Lambda(x|D^{(n_{T})})\|_{L_{\psi_{2}}} c_{K_0, b} +1
	\right)
	+1\right].
	\end{align}
\end{lemma}

\begin{proof}
	By definition of $\mathcal{Q}_{w}^{L_{2}}$, we have $\Diam\left(\mathcal{Q}_{w}^{L_{2}}(D^{(n_{T})}),\|\cdot\|_{L_{2}}\right)	=	2\sqrt{w}.$
	Using Lemma~\ref{lemma:covering_lipschitz} and \eqref{eq:bound_gamma}, we have
	\begin{align}
	\gamma_{2}\left(\mathcal{Q}_{w}^{L_{2}}(D^{(n_{T})}),\|\cdot\|_{L_{2}}\right)
	& \le	c\int_{0}^{2\sqrt{w}}\sqrt{\log N\left(\mathcal{Q}_{w}^{L_{2}}(D^{(n_{T})}),u,\|\cdot\|_{L_{2}}\right)}du \\
	& \le	c\int_{0}^{2\sqrt{w}}\sqrt{\log N\left(\Lambda,\frac{u}{c_0 \left(b +\sqrt{w}\right)\|C_\Lambda(x|D^{(n_T)})\|_{L_{2}}},\|\cdot\|_{2}\right)}du \\
	& \le	c\int_{0}^{2\sqrt{w}}\sqrt{J\log\left(\frac{4 c_0 \Delta_{\Lambda}\left(b +\sqrt{w}\right)\|C_\Lambda(x|D^{(n_T)})\|_{L_{2}}+2u}{u}\right)}du\\
	& \le	2c\sqrt{w J}\left[\sqrt{\log\left(\frac{4 c_0  \Delta_{\Lambda}\left(b +\sqrt{w}\right)\|C_\Lambda(x|D^{(n_T)})\|_{L_{2}}+4\sqrt{w}}{2\sqrt{w}}\right)}+\frac{\sqrt{\pi}}{2}\right]
	\end{align}
	Using very similar logic, we now bound the $\gamma_1$ function.
	First we bound the diameter of $\mathcal{Q}_w^{L_2}$ with respect to the norm $\|\cdot \|_{L_{\psi_1}}$:
	\begin{align}
	\Diam(\mathcal{Q}_w^{L_2}(D^{(n_T)}), \|\cdot \|_{L_{\psi_1}})
	& \le 2 \sup_{\boldsymbol{\lambda} \in \Lambda}
	\left \|
	\left(
	y - \hat{g}^{(n_T)}(\boldsymbol{\lambda}| D^{(n_T)})
	\right)^2
	-
	\left(
	y - g^*(x)
	\right)^2
	\right \|_{L_{\psi_1}}
	\le
	c_1 K_0.
	\label{eq:diam_psi1}
	\end{align}
	Thus
	\begin{align}
	\gamma_{1}\left(\mathcal{Q}_{w}^{L_{2}}(D^{(n_{T})}),\|\cdot\|_{L_{\psi_{1}}}\right)
	&\le c\int_{0}^{c_1 K_0}\log N\left(\mathcal{Q}_{w}^{L_{2}}(D^{(n_{T})}),u,\|\cdot\|_{L_{\psi_{1}}}\right)du\\
	& \le c_2 J K_0 \left[
	\log\left(
	\frac{4 \Delta_{\Lambda}c_{K_0, b} \|C_\Lambda(x|D^{(n_T)})\|_{L_{\psi_{2}}}+ 2 c_1 K_0}
	{c_1 K_0}
	\right)+1\right]
	\end{align}
\end{proof}

To apply Theorem~\ref{thrm:jean_cv}, we need to define $h$ and $J_\delta$ so that \eqref{eq:h_to_J} is satisfied.
Based on the lemma above, we see that it suffices to let
\begin{align}
h(D^{(n_T)}) \coloneqq \|C_\Lambda(x|D^{(n_T)})\|_{L_{\psi_2}}
\end{align}
and
\begin{align}
\begin{split}
\label{eq:j_delta}
\mathcal{J}_{\delta}(w) & =
c_{1} \frac{\log n_{V}}{\sqrt{n_{V}}}
JK_0\left[\log\left(\Delta_{\Lambda}\delta c_{K_0, b}+1\right)+1\right]
+c_{3}\sqrt{Jw}
\left[\sqrt{\log\left(\Delta_{\Lambda}b \delta n +1\right)}+1\right].
\end{split}
\end{align}
Finally using the results above, we can prove Theorem~\ref{thrm:kfold}.
\begin{proof}[Proof for Theorem~\ref{thrm:kfold}]
	We now apply Theorem~\ref{thrm:jean_cv} to our Lipschitz case.
	From \eqref{eq:diam_psi1}, we find that Assumption~\ref{assump:tail_margin_general} is satisfied.
	We have defined $h$ and $J_\delta$ so that \eqref{eq:h_to_J} is satisfied for all $w \ge 1/n$.
	Moreover, $\mathcal{J}_{\delta}(w)$ is strictly increasing and concave in $w$.
	This implies that $\mathcal{J}_{\delta}^{-1}$ is strictly convex.
	Via algebra, we find that the convex conjugate of $\mathcal{J}_{\delta}^{-1}$ is
	\begin{align}
	\psi_{\delta}(u)
	= c_{1} u \frac{\log n_{V}}{\sqrt{n_{V}}}
	JK_0\left[\log\left(\Delta_{\Lambda}\delta c_{K_0, b}+1\right)+1\right]
	+ u^2 c_{4} J
	\left[\sqrt{\log\left(\Delta_{\Lambda}b \delta n +1\right)}+1\right]^2.
	\end{align}
	Now let us determine $\tilde{\psi}_{q, \delta}(1/q)$ as $q \rightarrow 1$.
	We have
	\begin{align}
	\lim_{q\rightarrow1}\tilde{\psi}_{q,\delta}(1/q)
	& = \psi_{\delta}\left(\frac{2(1+a)}{a}\frac{1}{\sqrt{n_{V}}}\right)\vee\frac{1}{n_{V}}\\
	& \le
	c_{5} \left (\frac{1+a}{a} \right )^2 \frac{J\log n_{V}}{n_{V}}
	K_0\left[\log\left(\Delta_{\Lambda}\delta c_{K_0, b} n +1\right)+1\right].
	\end{align}
	So the summation in \eqref{eq:cv_oracle_ineq} reduces to
	\begin{align}
	& \lim_{q \rightarrow 1} \left(
	\tilde{\psi}_{q,2\sigma_0}(1/q)
	+\sum_{k=1}^{\infty}
	\Pr\left(h\left(D^{(n_{T})}\right)\ge2^{k}\sigma\right)
	\tilde{\psi}_{q,2^{k}\sigma_0}(1/q)
	\right)
	\\
	& \le
	c_{6} \left (\frac{1+a}{a} \right )^2 \frac{J\log n_{V}}{n_{V}}
	K_0\left[\log\left(\Delta_{\Lambda} c_{K_0, b} n \sigma_0 +1\right)+1\right]
	\left(
	1 + 
	\sum_{k=1}^{\infty}
	k \Pr\left(\|C_\Lambda(x|D^{(n_{T})})\|_{L_{\psi_{2}}} \ge 2^{k} \sigma_0\right)
	\right)
	\\
	& \le
	c_{6} \left (\frac{1+a}{a} \right )^2 \frac{J\log n_{V}}{n_{V}}
	K_0\left[\log\left(\Delta_{\Lambda} c_{K_0, b} n \sigma_0 +1\right)+1\right]
	\tilde{h}(n_{T}).
	\label{eq:sum_prob_bound}
	\end{align}
	Taking $q \rightarrow 1$ in \eqref{eq:cv_oracle_ineq} and plugging in \eqref{eq:sum_prob_bound} to Theorem~\ref{thrm:jean_cv}, we get our desired result.
\end{proof}

\subsection{Penalized regression for additive models}

We now show that penalized regression problems for additive models satisfy the Lipschitz condition.
\subsubsection{Proof for Lemma \ref{lemma:param_add}}
\begin{proof}
	We will use the notation $\hat{\boldsymbol{\theta}}(\boldsymbol{\lambda}) \coloneqq \hat{\boldsymbol{\theta}}(\boldsymbol{\lambda} | T)$. By the gradient optimality conditions, we have
	\begin{equation}
	\label{eq:grad_opt}
	\left.\nabla_{\theta} \left [
	\frac{1}{2}\left\Vert y-g(\boldsymbol{\theta})\right\Vert _{T}^{2}+
	\sum_{j=1}^J \lambda_{j}P_{j}(\boldsymbol{\theta}^{(j)})
	\right ]
	\right|_{\boldsymbol{\theta}=\hat{\boldsymbol{\theta}}(\lambda)}=0.
	\end{equation}
	
	After implicitly differentiating with respect to $\boldsymbol{\lambda}$, we have
	\begin{equation}
	\label{eq:implicit_diff}
	\nabla_{\lambda}\left\{ \left.\nabla_{\theta}
	\left [
	\frac{1}{2}\left\Vert y-g(\boldsymbol{\theta})\right\Vert _{T}^{2}
	+ \sum_{j=1}^J \lambda_{j}P_{j}(\boldsymbol{\theta}^{(j)})
	\right ]
	\right|_{\boldsymbol{\theta}=\hat{\boldsymbol{\theta}}(\boldsymbol{\lambda})}\right\} =0.
	\end{equation}
	From the product rule and chain rule, we can then write the system of equations in \eqref{eq:implicit_diff} as
	\begin{align}
	\label{eq:param_grad}
	\left . \nabla_{\lambda}\hat{\boldsymbol{\theta}}(\boldsymbol{\lambda})
	\right|_{\boldsymbol{\theta}=\hat{\boldsymbol{\theta}}(\lambda)}
	= -
	\left (
	\left.\nabla_{\theta}^2
	L_T(\boldsymbol{\theta}, \boldsymbol{\lambda})
	\right|_{\boldsymbol{\theta}=\hat{\boldsymbol{\theta}}(\lambda)} 
	\right)^{-1}
	\diag \left \{
	\left.
	\nabla_{\theta^{(j)}}P_{j}(\boldsymbol{\theta}^{(j)})
	\right|_{\boldsymbol{\theta}=\hat{\boldsymbol{\theta}}(\lambda)}
	\right \}_{j=1:J}
	.
	\end{align}
	We can bound the norm of the second term in \eqref{eq:param_grad} by rearranging \eqref{eq:grad_opt} and using the Cauchy-Schwarz inequality:
	\begin{align*}
	\left\Vert \left.\nabla_{\theta^{(j)}}P_{j}(\boldsymbol{\theta}^{(j)})\right|_{\boldsymbol{\theta}=\hat{\boldsymbol{\theta}}(\lambda)}\right\Vert_2
	&
	\le  \frac{1}{\lambda_{min}}\left\Vert y-g(\hat{\boldsymbol{\theta}}(\boldsymbol{\lambda}))\right\Vert _{T}
	\left \|
	\left\Vert
	\nabla_{\theta^{(j)}}g_{j}(x|\boldsymbol{\theta}^{(j)})
	\right\Vert_{2}
	\right \|_T.
	\end{align*}
	Since $g_j$ is Lipschitz by assumption, then
	\begin{align}
	\left\Vert
	\nabla_{\theta^{(j)}}g_{j}(x|\boldsymbol{\theta}^{(j)})
	\right\Vert_{2}
	\le \ell_j(x).
	\end{align}
	Also, by the definition of $\hat{\boldsymbol{\theta}}(\boldsymbol{\lambda})$, we have
	\begin{align}
	\frac{1}{2}\left\Vert y-g(\hat{\boldsymbol{\theta}}(\boldsymbol{\lambda}))\right\Vert _{T}^{2}
	& \le \frac{1}{2}\left\Vert \epsilon \right \Vert_T^2 + C^*_{\Lambda}.
	\end{align}
	Hence
	\begin{align}
	\left\Vert \left.\nabla_{\theta}P_{j}(\boldsymbol{\theta}^{(j)})\right|_{\boldsymbol{\theta}=\hat{\boldsymbol{\theta}}(\lambda)}\right\Vert_2
	& \le \frac{\|\ell_j\|_T}{\lambda_{min}}
	\sqrt{\left\Vert \epsilon \right \Vert_T^2 + 2 C^*_{\Lambda}}.
	\end{align}
	Plugging in the results from above and using the assumption that the Hessian of the objective function has a minimum eigenvalue of $m(T)$, we have for all 
	\begin{align}
	\left .
	\nabla_{\lambda_k}\hat{\boldsymbol{\theta}}^{(j)}(\boldsymbol{\lambda})
	\right|_{\boldsymbol{\theta}=\hat{\boldsymbol{\theta}}(\lambda)}
	& = \boldsymbol{0}
	\text{ if } j\ne k
	\\
	\left \|
	\left .
	\nabla_{\lambda_j}\hat{\boldsymbol{\theta}}^{(j)}(\boldsymbol{\lambda})
	\right|_{\boldsymbol{\theta}=\hat{\boldsymbol{\theta}}(\lambda)}
	\right \|_2
	& = \left \|
	\left .
	\nabla_{\lambda_j}\hat{\boldsymbol{\theta}}(\boldsymbol{\lambda})
	\right|_{\boldsymbol{\theta}=\hat{\boldsymbol{\theta}}(\lambda)}
	\right \|_2
	\\
	& \le \frac{1}{m(T)} \frac{\|\ell_j\|_T}{\lambda_{min}}
	\sqrt{\left\Vert \epsilon \right \Vert_T^2 + 2 C^*_{\Lambda}}.
	\end{align}
	Since the norm of the gradient is bounded, $\hat{\boldsymbol{\theta}}^{(j)}(\boldsymbol{\lambda})$ must be Lipschitz:
	\begin{align}
	\left\Vert \hat{\boldsymbol{\theta}}^{(j)}(\boldsymbol{\lambda}^{(1)})
	-\hat{\boldsymbol{\theta}}^{(j)}(\boldsymbol{\lambda}^{(2)})\right\Vert _{2}
	& \le
	\frac{1}{m(T)} \frac{\|\ell_j\|_T}{\lambda_{min}}
	\sqrt{\left\Vert \epsilon \right \Vert_T^2 + 2 C^*_{\Lambda}}
	\left |{\lambda}^{(1)}_j-{\lambda}^{(2)}_j \right |.
	\label{eq:lipschitz_params}
	\end{align}
	Finally we combine the above results to get
	\begin{align}
	& \left |
	g \left (x \middle | \hat{\boldsymbol{\theta}}(\boldsymbol{\lambda}^{(1)})\right )
	- g \left (x \middle | \hat{\boldsymbol{\theta}}(\boldsymbol{\lambda}^{(2)})\right )
	\right |\\
	& \le
	\sum_{j=1}^J
	\left |
	g_j \left (x \middle | \hat{\boldsymbol{\theta}}(\boldsymbol{\lambda}^{(1)})\right )
	- g_j \left (x \middle | \hat{\boldsymbol{\theta}}(\boldsymbol{\lambda}^{(2)})\right )
	\right | \\
	& \le \sum_{j=1}^J \ell_j(x_j)
	\left\Vert \hat{\boldsymbol{\theta}}^{(j)}(\boldsymbol{\lambda}^{(1)})
	-\hat{\boldsymbol{\theta}}^{(j)}(\boldsymbol{\lambda}^{(2)})\right\Vert _{2}\\
	& \le \sum_{j=1}^J \ell_j(x_j)
	\frac{1}{m(T)} \frac{\|\ell_j\|_T}{\lambda_{min}}
	\sqrt{\left\Vert \epsilon \right \Vert_T^2 + 2 C^*_{\Lambda}}
	\left |{\lambda}^{(1)}_j-{\lambda}^{(2)}_j\right |\\
	& \le
	\frac{1}{m(T) \lambda_{min}}
	\sqrt{
		\left(
		\left\Vert \epsilon \right \Vert_T^2 + 2 C^*_{\Lambda}
		\right)
		\left(
		\sum_{j=1}^J \|\ell_j\|_T^2 \ell_j^2(x_j)
		\right)
	}
	\left \|
	\boldsymbol{\lambda}^{(1)}-\boldsymbol{\lambda}^{(2)}
	\right \|_{2}
	\end{align}
\end{proof}

\subsubsection{Proof for Lemma \ref{lemma:nonsmooth}}

Before proving Lemma~\ref{lemma:nonsmooth}, we need to introduce some notation.
Let $\mathcal{L}(\boldsymbol{\lambda}^{(1)},\boldsymbol{\lambda}^{(2)})$
be the line segment connecting $\boldsymbol{\lambda}^{(1)}$ and $\boldsymbol{\lambda}^{(2)}$.
Let $\mu_{1}(z)$ be the 1-dimensional Lebesgue measure in the direction
of $z$ (so if $z$ is a continuous line segment, $\mu_{1}(z)=\|z\|_{2}$;
if $z$ is composed of multiple line segments $z_{i}$, then $\mu(z)=\sum\mu(z_{i})$).

Before proving the Lipschitz property over all of $\Lambda$, we show that the fitted function is Lipschitz over $\Lambda_{smooth}$.
For convenience, define $\Lambda_{smooth}^{c} \coloneqq \Lambda \setminus \Lambda_{smooth}$.

\begin{lemma}
	Suppose that $g_j(\boldsymbol{\theta})(x)$ satisfies the Lipschitz condition in Lemma~\ref{lemma:param_add}.
	Let $T\equiv D^{(n_{T})}$ be a fixed set of training data. Suppose
	the penalized loss function $L_{T}\left(\boldsymbol{\theta},\boldsymbol{\lambda}\right)$
	has a unique minimizer $\hat{\boldsymbol{\theta}}(\boldsymbol{\lambda}|T)$
	for every $\boldsymbol{\lambda}\in\Lambda$. Let $\boldsymbol{U}_{\lambda}$
	be an orthonormal matrix with columns forming a basis for the differentiable
	space of $L_{T}(\cdot,\boldsymbol{\lambda})$ at $\hat{\boldsymbol{\theta}}(\boldsymbol{\lambda}|T)$.
	Suppose there exists a constant $m(T)>0$ such that the Hessian of
	the penalized training criterion at the minimizer taken with respect
	to the directions in $\boldsymbol{U}_{\lambda}$ satisfies 
	\begin{equation}
	\left._{U_{\lambda}}\nabla_{\theta}^{2}L_{T}(\boldsymbol{\theta},\boldsymbol{\lambda})\right|_{\theta=\hat{\theta}(\boldsymbol{\lambda})}\succeq m(T)\boldsymbol{I}\quad\forall\boldsymbol{\lambda}\in\Lambda
	\end{equation}
	where \textup{$\boldsymbol{I}$ is the identity matrix.}
	Suppose Condition~\ref{condn:nonsmooth1} is satisfied by some $\Lambda_{smooth}\subseteq\Lambda$.
	Define
	\begin{align}
	\Lambda_{ext}=\left\{ (\boldsymbol{\lambda}^{(1)},\boldsymbol{\lambda}^{(2)}):
	\boldsymbol{\lambda}^{(1)},\boldsymbol{\lambda}^{(2)} \in \Lambda,
	\mu_{1}\left(\mathcal{L}(\boldsymbol{\lambda}^{(1)},\boldsymbol{\lambda}^{(2)})\cap\Lambda_{smooth}^{c}\right)>0\right\}.
	\label{eq:lambda_ext}
	\end{align}
	Then any $(\boldsymbol{\lambda}^{(1)},\boldsymbol{\lambda}^{(2)})\in\Lambda_{ext}^{c}$
	satisfies \eqref{eq:param_add_lipschitz}.
	\label{lemma:lipschitz_lambda_ext_c}
\end{lemma}
\begin{proof}
	From Condition 1, every point $\boldsymbol{\lambda}\in\Lambda_{smooth}$
	is the center of a ball $B(\boldsymbol{\lambda})$ with nonzero radius
	where the differentiable space within $B(\boldsymbol{\lambda})$ is
	constant.
	
	Now consider any $\boldsymbol{\lambda^{(1)}},\boldsymbol{\lambda^{(2)}}\in \Lambda_{ext}$.
	By \eqref{eq:lambda_ext}, there must exist a countable set of points $\cup_{i=1}^{\infty}\boldsymbol{\ell}^{(i)}
	\subset\mathcal{L}(\boldsymbol{\lambda^{(1)}},\boldsymbol{\lambda^{(2)}})$ where
	$\cup_{i=1}^{\infty} \boldsymbol{\ell}^{(i)} \subset \Lambda_{smooth}$,
	$\boldsymbol{\lambda}^{(1)},\boldsymbol{\lambda}^{(2)}\in \cup_{i=1}^{\infty}\boldsymbol{\ell}^{(i)}$,
	and the union of their differentiable neighborhoods cover $\mathcal{L}(\boldsymbol{\lambda^{(1)}},\boldsymbol{\lambda^{(2)}})$
	entirely: 
	\[
	\mathcal{L}\left(\boldsymbol{\lambda^{(1)}},\boldsymbol{\lambda^{(2)}}\right)\subseteq\cup_{i=1}^{\infty}B\left(\boldsymbol{\ell}^{(i)}\right).
	\]
	Consider the intersections of boundaries of the differentiable neighborhoods
	with the line segment: 
	\begin{align}
	P =
	\cup_{i=1}^{\infty}
	\left[
	bd \left (
	B\left(\boldsymbol{\ell}^{(i)}\right)
	\right )
	\cap\mathcal{L}(\boldsymbol{\lambda^{(1)}},\boldsymbol{\lambda^{(2)}})
	\right].
	\end{align}
	Every point $p\in P$ can be expressed as $\alpha_{p}\boldsymbol{\lambda^{(1)}}+(1-\alpha_{p})\boldsymbol{\lambda^{(2)}}$
	for some $\alpha_{p}\in[0,1]$.
	We can order the points in $P$ by increasing $\alpha_{p}$ to get the sequence $\boldsymbol{p}^{(1)},\boldsymbol{p}^{(2)},...$.
	
	By Condition 1, the differentiable space of the training criterion
	is constant over $\mathcal{L}\left(\boldsymbol{p}^{(i)},\boldsymbol{p}^{(i+1)}\right)$
	since each of these sub-segments are contained in some $B(\boldsymbol{\ell}^{(i)})$
	for $i\in\mathbb{N}$.
	Moreover, the differentiable space over the interior of line segment $\mathcal{L}\left(\boldsymbol{p^{(i)},p^{(i+1)}}\right)$ can be decomposed as the product of differentiable spaces, which we denote as
	\begin{align}
	\Omega_{i}^{(1)} \times ... \times \Omega_{i}^{(J)}.
	\label{eq:diff_space_product}
	\end{align}
	By Condition 1, \eqref{eq:diff_space_product} is also a local optimality space.
	Let $U^{(i,j)}$ be an orthonormal basis of $\Omega_{i}^{(j)}$ for $j = 1,...,J$.
	For each $i$, we can express $\hat{\boldsymbol{\theta}}(\boldsymbol{\lambda}|T)$
	for all $\boldsymbol{\lambda}\in\mbox{Int}\left\{ \mathcal{L}\left(\boldsymbol{p^{(i)},p^{(i+1)}}\right)\right\} $
	as 
	\[
	\hat{\boldsymbol{\theta}}^{(j)}(\boldsymbol{\lambda}|T)
	=U^{(i,j)}\hat{\boldsymbol{\beta}}^{(j)}(\boldsymbol{\lambda}|T)
	\]
	\[
	\hat{\boldsymbol{\beta}}(\boldsymbol{\lambda}|T)
	= \left (
	\begin{matrix}
	\hat{\boldsymbol{\beta}}^{(1)}(\boldsymbol{\lambda}|T)
	& ... &
	\hat{\boldsymbol{\beta}}^{(J)}(\boldsymbol{\lambda}|T)
	\end{matrix}
	\right )
	=\arg\min_{\beta}L_{T}
	\left (
	\{U^{(i,j)}\boldsymbol{\beta}^{(j)} \}_{j=1}^J,
	\boldsymbol{\lambda}
	\right ).
	\]
	We can show that the fitted parameters satisfy the Lipschitz condition \eqref{eq:lipschitz_params} over $\Lambda=\mathcal{L}\left(\boldsymbol{p^{(i)},p^{(i+1)}}\right)$ by using a similar proof as in Lemma~\ref{lemma:param_add}.
	The only difference is that the proofs starts with taking directional derivatives along the columns of $U^{(i)} = (U^{(i,1)} ... U^{(i,J)})$ to establish the KKT conditions.
	Then for all $j$ and $i$, we have
	\begin{align}
	\left\Vert
	\boldsymbol{\hat{\beta}}^{(j)}(\boldsymbol{p}^{(i)}|T)
	-\boldsymbol{\hat{\beta}}^{(j)}(\boldsymbol{p}^{(i)}|T)
	\right\Vert _{2}
	\le
	\frac{1}{m(T)} \frac{\|\ell_j\|_T}{\lambda_{min}}
	\sqrt{\left\Vert \epsilon \right \Vert_T^2 + 2 C^*_{\Lambda}}
	\left |p^{(i)}_j-p^{(i+1)}_j\right |.
	\end{align}
	We can sum these inequalities by the triangle inequality:
	\begin{align*}
	\left\Vert
	\hat{\boldsymbol{\theta}}^{(j)}(\boldsymbol{\lambda}^{(1)}|T)
	-\hat{\boldsymbol{\theta}}^{(j)}(\boldsymbol{\lambda}^{(2)}|T)
	\right\Vert _{2}
	& \le
	\sum_{i=1}^{\infty}
	\left \|
	\boldsymbol{\hat{\theta}}^{(j)}(\boldsymbol{p}^{(i)}|T)
	-\boldsymbol{\hat{\theta}}^{(j)}(\boldsymbol{p}^{(i + 1)}|T)
	\right \|_{2}\\
	& \le
	\frac{1}{m(T)} \frac{\|\ell_j\|_T}{\lambda_{min}}
	\sqrt{\left\Vert \epsilon \right \Vert_T^2 + 2 C^*_{\Lambda}}
	\sum_{i=1}^{\infty} 
	\left |
	{p}^{(i)}_j-{p}^{(i+1)}_j
	\right |\\
	& =
	\frac{1}{m(T)} \frac{\|\ell_j\|_T}{\lambda_{min}}
	\sqrt{\left\Vert \epsilon \right \Vert_T^2 + 2 C^*_{\Lambda}}
	\left |{\lambda}^{(1)}_j - {\lambda}^{(2)}_j\right |.
	\end{align*}
	Finally, using the fact that $g_j$ is $\ell_j$-Lipschitz, we have by the triangle inequality and Cauchy Schwarz that
	\begin{align}
	C_\Lambda(\boldsymbol{x}|T)
	= \frac{\sqrt{\| \epsilon \|_T^2 + 2 C^*_{\Lambda}}}{m(T) \lambda_{min}}
	\sqrt{\sum_{j=1}^J \|\ell_j\|_T^2 \ell_j^2(x_j)}.
	\label{eq:nonsmooth_lipschitz_func}
	\end{align}
\end{proof}

In order to extend the result in Lemma~\ref{lemma:lipschitz_lambda_ext_c} to all of $\Lambda$, we need to show that $\Lambda_{ext}$ is a set with measure zero.
\begin{lemma}
	Suppose Condition~\ref{condn:nonsmooth2}.
	Then $\mu_{2J}(\Lambda_{ext})=0$ where $\mu_{2J}$ is the Lebesgue measure in $\mathbb{R}^{2J}$ and $\Lambda_{ext}$ was defined in \eqref{eq:lambda_ext}.
	\label{lemma:ext_measure_zero}
\end{lemma}

\begin{proof}
	Suppose for contradiction that $\mu_{2J}(\Lambda_{ext})>0$.
	If this is the case, then there exists a ball $B_{r}\left(\left(\boldsymbol{\lambda}^{(1)},\boldsymbol{\lambda}^{(2)}\right)\right)$ contained in $\Lambda_{ext}$ with nonzero radius $r>0$ centered at $\left(\boldsymbol{\lambda}^{(1)},\boldsymbol{\lambda}^{(2)}\right)$
	where $\boldsymbol{\lambda}^{(1)} \ne \boldsymbol{\lambda}^{(2)}$ and
	\begin{align}
	\mu_{1}\left(\mathcal{L}\left(\boldsymbol{\lambda}^{'},\boldsymbol{\lambda}^{''}\right)\cap\Lambda_{smooth}^{c}\right) >0
	\quad \forall\left(\boldsymbol{\lambda}^{'},\boldsymbol{\lambda}^{''}\right)\in B_{r}\left(\left(\boldsymbol{\lambda}^{(1)},\boldsymbol{\lambda}^{(2)}\right)\right).
	\end{align}
	Suppose that $\mu_{1}\left(\mathcal{L}\left(\boldsymbol{\lambda}^{(1)},\boldsymbol{\lambda}^{(2)}\right)\cap\Lambda_{smooth}^{c}\right)=\delta>0$.
	We claim that for a sufficiently small radius $r'$, we also have
	\begin{align}
	\mu_{1}\left(\mathcal{L}\left(\boldsymbol{\lambda}^{'},\boldsymbol{\lambda}^{''}\right)\cap\Lambda_{smooth}^{c}\right)>\delta/2>0
	\quad \forall
	\left(\boldsymbol{\lambda}^{'},\boldsymbol{\lambda}^{''}\right)
	\in B_{r'}\left(\left(\boldsymbol{\lambda}^{(1)},\boldsymbol{\lambda}^{(2)}\right)\right).
	\end{align}
	To see why this claim is true, let us define a monotonically decreasing
	sequence $\left\{ r_{i}\right\} $ where $r_{i} > 0$ for all $i \in \mathbb{N}$ and $\lim_{i\rightarrow\infty}r_{i}=0$.
	By the monotone convergence theorem,
	\begin{align}
	\lim_{i\rightarrow\infty}\inf_{\left(\boldsymbol{\lambda}^{'},\boldsymbol{\lambda}^{''}\right)\in B_{r_{i}}\left(\left(\boldsymbol{\lambda}^{(1)},\boldsymbol{\lambda}^{(2)}\right)\right)}\mu_{1}\left(\mathcal{L}\left(\boldsymbol{\lambda}^{'},\boldsymbol{\lambda}^{''}\right)\cap\Lambda_{smooth}^{c}\right)=\mu_{1}\left(\mathcal{L}\left(\boldsymbol{\lambda}^{(1)},\boldsymbol{\lambda}^{(2)}\right)\cap\Lambda_{smooth}^{c}\right)=\delta>0.
	\end{align}
	By the definition of limits, there is some sufficiently
	large $i'$ such that for $r'\coloneqq r_{i'} > 0$, we have
	\begin{align}
	\inf_{\left(\boldsymbol{\lambda}^{'},\boldsymbol{\lambda}^{''}\right)\in B_{r'}\left(\left(\boldsymbol{\lambda}^{(1)},\boldsymbol{\lambda}^{(2)}\right)\right)}\mu_{1}\left(\mathcal{L}\left(\boldsymbol{\lambda}^{'},\boldsymbol{\lambda}^{''}\right)\cap\Lambda_{smooth}^{c}\right)>\delta/2.
	\end{align}
	Given our ball is non-empty, there exist points $\left(\boldsymbol{\lambda}^{(3)},\boldsymbol{\lambda}^{(4)}\right),\left(\boldsymbol{\lambda}^{(5)},\boldsymbol{\lambda}^{(6)}\right)\in B_{r'}\left(\left(\boldsymbol{\lambda}^{(1)},\boldsymbol{\lambda}^{(2)}\right)\right)$
	where 
	\begin{align}
	{\lambda}_{j}^{(3)} > {\lambda}_{j}^{(5)},
	{\lambda}_{j}^{(4)} > {\lambda}_{j}^{(6)}
	\quad \forall j=1,..,J.
	\end{align}
	For any $\alpha \in (0,1)$, the line
	\begin{align}
	\mathcal{L}_\alpha =
	\mathcal{L}\left(\alpha\boldsymbol{\lambda}^{(3)}+(1-\alpha)\boldsymbol{\lambda}^{(5)},\alpha\boldsymbol{\lambda}^{(4)}+(1-\alpha)\boldsymbol{\lambda}^{(6)}\right)
	\end{align}
	has
	\begin{align}
	\mu_1\left(
	\mathcal{L}_\alpha
	\cap
	\Lambda_{smooth}^{c}
	\right)
	> \delta/2.
	\end{align}
	As the lines $\mathcal{L}_\alpha$ do not intersect for $\alpha \in (0,1)$, then
	\begin{align}
	\mu\left(
	\cup_{\alpha\in[0,1]}
	\left(
	\mathcal{L}_\alpha
	\cap\Lambda_{smooth}^{c}
	\right)\right)
	= \int_{0}^{1} \mu_1\left(
	\mathcal{L}_\alpha
	\cap
	\Lambda_{smooth}^{c}
	\right) d\alpha
	> \delta/2
	\end{align}
	Thus
	\begin{align}
	\mu\left(\Lambda_{smooth}^{c}\right)
	\ge
	\mu\left(
	\cup_{\alpha\in[0,1]}
	\left(
	\mathcal{L}_\alpha
	\cap\Lambda_{smooth}^{c}
	\right)\right)
	>\delta/2.
	\end{align}
	However this is a contradiction of our assumption that $\mu\left(\Lambda_{smooth}^{c}\right)=0$.
\end{proof}

Finally, combining Lemmas~\ref{lemma:lipschitz_lambda_ext_c} and \ref{lemma:ext_measure_zero}, we can show that the Lipschitz condition is satisfied over all of $\Lambda$.

\begin{proof}[Proof for Lemma~\ref{lemma:nonsmooth}]
	Since we already showed Lemma~\ref{lemma:lipschitz_lambda_ext_c}, it suffices to show that the Lipschitz condition is satisfied for any $\boldsymbol{\lambda^{(1)}},\boldsymbol{\lambda^{(2)}}\in\Lambda_{ext}$.
	Lemma~\ref{lemma:ext_measure_zero} states that $\mu_{2J}(\Lambda_{ext}) = 0$, which means that there exists a sequence
	$\left \{\left(\boldsymbol{\lambda}^{(1,i)},\boldsymbol{\lambda}^{(2,i)}\right) \right \}_{i=1}^\infty \subseteq \Lambda_{ext}^{c}$
	such that $\lim_{i\rightarrow\infty}\left(\boldsymbol{\lambda}^{(1,i)},\boldsymbol{\lambda}^{(2,i)}\right)=\left(\boldsymbol{\lambda}^{(1)},\boldsymbol{\lambda}^{(2)}\right)$.
	As $L_T$ is continuous and we have assumed that there exists a unique minimizer of $\hat{\boldsymbol{\theta}}(\boldsymbol{\lambda})$ for all $\boldsymbol{\lambda} \in \Lambda$, then $\hat{\boldsymbol{\theta}}(\boldsymbol{\lambda})$ is continuous in $\boldsymbol{\lambda}$ over all $\Lambda$.
	As $g(\boldsymbol{\theta})(x)$ is also continuous in $\boldsymbol{\theta}$, then for any $\boldsymbol{\lambda}^{(1)},\boldsymbol{\lambda}^{(2)}\in\Lambda$,
	we have
	\begin{align}
	\left |
	g(\hat{\boldsymbol{\theta}}(\boldsymbol{\lambda}^{(1)} | T)(\boldsymbol{x})
	-
	g(\hat{\boldsymbol{\theta}}(\boldsymbol{\lambda}^{(2)} | T)(\boldsymbol{x})
	\right |
	& = \lim_{i\rightarrow\infty}
	\left |
	g(\hat{\boldsymbol{\theta}}(\boldsymbol{\lambda}^{(1,i)} |T))(\boldsymbol{x})
	-
	g(\hat{\boldsymbol{\theta}}(\boldsymbol{\lambda}^{(2,i)} |T))(\boldsymbol{x})
	\right |\\
	& \le \lim_{i\rightarrow\infty}
	C_\Lambda(\boldsymbol{x}|T)
	\|\boldsymbol{\lambda}^{(1,i)}-\boldsymbol{\lambda}^{(2,i)}\|_{2}\\
	& = C_\Lambda(\boldsymbol{x}|T)
	\|\boldsymbol{\lambda}^{(1)}-\boldsymbol{\lambda}^{(2)}\|_{2}
	\end{align}
	where $C_\Lambda(\boldsymbol{x}|T)$ is defined in \eqref{eq:nonsmooth_lipschitz_func}.
\end{proof}

\subsubsection{Proof for Lemma \ref{lemma:nonparam_smooth}}

\begin{proof}
	Let $H_{0} = \left \{
	j:\left\Vert \hat{g}_{j}(\boldsymbol{\lambda}^{(2)}|T)-\hat{g}_{j}(\boldsymbol{\lambda}^{(1)}|T)\right\Vert _{D^{(n)}} \ne 0\,\, \forall j = 1,...,J
	\right \}$.
	For all $j \in H_0$, let 
	\[
	h_{j}=
	\frac{\hat{g}_{j}(\boldsymbol{\lambda}^{(2)}|T)-\hat{g}_{j}(\boldsymbol{\lambda}^{(1)}|T)}{\left\Vert \hat{g}_{j}(\boldsymbol{\lambda}^{(2)}|T)-\hat{g}_{j}(\boldsymbol{\lambda}^{(1)}|T)\right\Vert _{D^{(n)}}}.
	\]
	For notational convenience, let $\hat{g}_{1,j} = \hat{g}_{j}(\boldsymbol{\lambda}^{(1)}|T)$. Consider the optimization problem
	\begin{equation}
	\hat{\boldsymbol{m}}(\boldsymbol{\lambda})=\left\{ \hat{m}_{j}(\boldsymbol{\lambda})\right\} _{j\in H_0}
	=\argmin_{m_{j} \in \mathbb{R}: j\in H_0}
	\frac{1}{2}
	\left \|y-\sum_{j=1}^{J}\left(\hat{g}_{1,j}+m_{j}h_{j}\right) \right \|_{T}^{2}
	+\sum_{j=1}^{J}\lambda_{j}
	P_{j} \left (\hat{g}_{1,j}+m_{j}h_{j} \right ).
	\end{equation}
	By the gradient optimality conditions, we have
	\begin{equation}
	\nabla_{m} \left .
	\left[\frac{1}{2}\|y-\sum_{j=1}^{J}\left(\hat{g}_{1,j}+m_{j}h_{j}\right)\|_{T}^{2}+\sum_{j=1}^{J}\lambda_{j}P_{j}(\hat{g}_{1,j}+m_{j}h_{j})\right] \right |_{m=\hat{m}(\lambda)}
	= 0.
	\label{eq:nonparam_grad_opt}
	\end{equation}
	Implicit differentiation with respect to $\boldsymbol{\lambda}$ gives us
	\begin{equation}
	\nabla_\lambda 
	\nabla_m
	\left . \left[
	\frac{1}{2}\|y-\sum_{j=1}^{J}\left(\hat{g}_{1,j}+m_{j}h_{j}\right)\|_{T}^{2}+\sum_{j=1}^{J}\lambda_{j}P_{j}(\hat{g}_{1,j}+m_{j}h_{j})\right] \right |_{m=\hat{m}(\lambda)}
	= 0.
	\label{eq:nonparam_imp_diff}
	\end{equation}
	From the product rule and chain rule, we can write the system of equations from \eqref{eq:nonparam_imp_diff} as
	\begin{align}
	\nabla_{\lambda} \hat{\boldsymbol{m}}(\boldsymbol{\lambda}) = - \left(
	\nabla_{m}^2 L_T(\boldsymbol{m}, \boldsymbol{\lambda})
	\right )^{-1}
	\diag \left \{ \left.
	\frac{\partial}{\partial m_{j}}P_{j}(\hat{g}_{1,j}+m_{j}h_{j})\right|_{m=\hat{m}(\lambda)}
	\right \}_{j=1}^J
	\label{eq:nonparam_grad}
	\end{align}
	where $L_T(\boldsymbol{m}, \boldsymbol{\lambda})$ is the loss in \eqref{eq:nonparam_grad_opt}.
	
	We now bound the second term in \eqref{eq:nonparam_grad}.
	From \eqref{eq:nonparam_grad_opt} and Cauchy Schwarz, we have for all $k=1,...,J$
	\begin{equation}
	\left|\frac{\partial}{\partial m_{k}}P_{k}(\hat{g}_{1,k}+m_{k}h_{k})\right|_{m=\hat{m}(\lambda)}
	\le 
	\frac{1}{\lambda_{min}}
	\left\Vert y-\sum_{j=1}^{J}\left(\hat{g}_{1,j}+\hat{m}_{j}(\boldsymbol{\lambda})h_{j}\right)
	\right\Vert _{T}\|h_{k}\|_{T}.
	\end{equation}
	From the definition of $h_k$, we know that $\|h_{k}\|_{T} \le \sqrt{\frac{n_{D}}{n_{T}}}$.
	By definition of $\hat{m}(\boldsymbol{\lambda})$ and $\hat{g}_{1}$, we also have
	\begin{align*}
	\frac{1}{2}\left\Vert y-\sum_{j=1}^{J}\left(\hat{g}_{1,j}+\hat{m}_{j}(\boldsymbol{\lambda})h_{j}\right)\right\Vert _{T}^{2}
	\le
	\frac{1}{2}
	\left\Vert y-\sum_{j=1}^{J}\hat{g}_{1,j}\right\Vert _{T}^{2}
	+\sum_{j=1}^{J}\lambda_{j}P_{j}(\hat{g}_{1,j})
	\le \frac{1}{2} \|\epsilon\|_T^2 + C^*_\Lambda.
	\end{align*}
	Hence
	\begin{equation}
	\left|\frac{\partial}{\partial m_{k}}P_{k}(\hat{g}_{1,k}+m_{k}h_{k})\right|_{m=\hat{m}(\lambda)}
	\le
	\frac{1}{\lambda_{min}}
	\sqrt{
		\left(
		\|\epsilon\|_T^2 + 2 C^*_\Lambda
		\right)
		\frac{n_{D}}{n_{T}}
	}.
	\end{equation}
	By \eqref{eq:gateuax}, we know $\nabla_{m}^2 L_T(\boldsymbol{m}, \boldsymbol{\lambda}) \succeq m(T)I$.
	So for all $k$,
	\begin{align}
	\|\nabla_{\lambda}\hat{m}_{k}(\boldsymbol{\lambda})\|_2
	& \le
	\frac{m(T)}{\lambda_{min}}
	\sqrt{
		\left(
		\|\epsilon\|_T^2 + 2 C^*_\Lambda
		\right)
		\frac{n_{D}}{n_{T}}
	}
	\end{align}
	By the mean value inequality and Cauchy Schwarz, we have
	\begin{equation}
	\left|\hat{m}_{k}(\boldsymbol{\lambda}^{(2)})-\hat{m}_{k}(\boldsymbol{\lambda}^{(1)})\right| 
	\le
	\frac{m(T)}{\lambda_{min}}
	\sqrt{
		\left(
		\|\epsilon\|_T^2 + 2 C^*_\Lambda
		\right)
		\frac{n_{D}}{n_{T}}
	}.
	\end{equation}
	By construction,
	$
	\left|
	\hat{m}_k(\boldsymbol{\lambda}^{(2)})-\hat{m}_k(\boldsymbol{\lambda}^{(1)})
	\right|  =
	\left \| 
	\hat{g}_k(\boldsymbol{\lambda}^{(2)}|T)-\hat{g}_k(\boldsymbol{\lambda}^{(1)}|T)
	\right  \|_{D^{(n)}}
	$.
	So we obtain our desired result in \eqref{eq:nonparam_lipshitz_thrm}.
\end{proof}

\subsection{Examples: detailed derivations}

\noindent \textbf{Example~\ref{ex:ridge}} (Multiple ridge penalties)
Here we present the details for deriving \eqref{eq:param_add_lipschitz} for Example \ref{ex:ridge}.
The additive components $g_j(\boldsymbol{\theta}^{(j)})(\boldsymbol{x}^{(j)})$ are linear functions that are $\ell_j$-Lipschitz where $\ell_j(\boldsymbol{x}^{(j)}) = \|\boldsymbol{x}^{(j)}\|_2$.
Then by Lemma~\ref{lemma:param_add}, the fitted function $g(\hat{\boldsymbol{\theta}}(\boldsymbol{\lambda}))(\boldsymbol{x})$ satisfy Assumption~\ref{assump:lipschitz} over $\mathbb{R}^p$ with
\begin{align}
C_\Lambda \left ( \boldsymbol{x} | T \right ) =
n^{2t_{\min}}
\sqrt{
	C^*_{T}
	\left(
	\sum_{j = 1}^J
	\|\boldsymbol{x}^{(j)}\|_2^2
	\left(\frac{1}{n_T} \sum_{(x_i, y_i) \in T} \|\boldsymbol{x}_i^{(j)}\|^2_2\right)
	\right)
}
\label{eq:ex_ridge}
\end{align}
where $C^*_{T}$ is defined in Example~\ref{ex:ridge} of the main manuscript.

\noindent \textbf{Example~\ref{example:sobolev}} (Multiple sobolev penalties)
here we present the details for deriving \eqref{eq:param_add_lipschitz} for Example~\ref{example:sobolev}
Since the solution to \eqref{eq:smoothing_spline} must be the sum of natural cubic splines \citep{buja1989linear}, we can parameterize the space using a Reproducing Kernel Hilbert Space with inner product
\begin{align}
\langle f, g \rangle = \int_{0}^1 f^{''}(x) g^{''}(x) dx
\end{align}
and the reproducing kernel
\begin{align}
R(s, t) = st(s \wedge t)
+ \frac{s + t}{2} (s \wedge t)^2
+ \frac{1}{3}
(s \wedge t)^3
\end{align}
\citep{heckman2012theory}.
Then one can instead solve for \eqref{eq:smoothing_spline} over the functions $g$ of the form
\begin{align}
g(x_1,..., x_J) = \alpha_0 + \sum_{j=1}^J g_j(x_j)
\end{align}
where the functions $g_j$ are split into a linear component and an orthogonal non-linear component
\begin{align}
g_j(x_j) = \alpha_{1j} x_j + \sum_{i=1}^{n_T} \theta_{ij} R(x_{ij}, x_j).
\end{align}
For notational simplicity, we will also denote $\vec{R}(x | D)_{ij} = R(x_{ij}, x_j)$.
We will also write
\begin{align}
g_{j, \perp}(x_j) = \sum_{i=1}^{n_T} \theta_{ij} R(x_{ij}, x_j).
\end{align}

Using this finite-dimensional representation, we find that
\begin{align}
\int_{0}^1 \left(g_j^{''}(x)\right)^{2} dx
= \sum_{u = 1}^{n_T} \sum_{v=1}^{n_T} \theta_{uj} \theta{vj} R(x_{uj}, x_{vj})
= \theta_j^\top K_j \theta_j
\label{eq:sobolev_finite}
\end{align}
where the matrix $K_j$ has elements
$
K_{j, (u, v)} = R(x_{uj}, x_{vj}).
$
Since any $g_j$ with non-zero $\boldsymbol{\theta}_j$ will have a positive Sobolev penalty, then the matrix $K_j$ must be positive definite.
Using the formulation above, we re-express \eqref{eq:smoothing_spline} as the finite-dimensional problem
\begin{align}
\hat{\alpha_0}(\boldsymbol{\lambda}),
\hat{\boldsymbol{\alpha}_1}(\boldsymbol{\lambda}),
\hat{\boldsymbol{\theta}}(\boldsymbol{\lambda})
& = \argmin_{\alpha_0, \boldsymbol{\alpha}_1, \boldsymbol{\theta}}
\frac{1}{2}
\left \|
\boldsymbol{y}_T -
\alpha_0 \boldsymbol{1}
- X_T \boldsymbol{\alpha}_1
- K \boldsymbol{\theta}
\right \|^2_2
+
\frac{1}{2}
\boldsymbol{\theta}^\top
\diag \left (
\left \{
\lambda_j K_j
\right \} \right ) \boldsymbol{\theta}.
\end{align}
where $K = (K_1 ... K_J)$.
In order to make the fitted functions $\hat{g}_j$ identifiable, we add the usual constraint that $\sum_{i=1}^{n_T} g_j(x_{ij}) = 0$ for all $j$.
We also assume that $X_T^\top X_T$ is nonsingular to ensure that there is a unique $\hat{\alpha}_1$.

The KKT conditions then gives us
\begin{align}
\hat{\alpha}_0 &= \frac{1}{n_T}\sum_{(x_i, y_i)\in T} y_i \\
\hat{\boldsymbol{\alpha}}_1(\boldsymbol{\lambda})
& = (X_T^\top X_T)^{-1} X_T^\top
(
\boldsymbol{y}_T - \hat{\alpha}_0 \boldsymbol{1}
- K \hat{\boldsymbol{\theta}}(\boldsymbol{\lambda})
)
\label{eq:kkt_sobolev_linear}
\\
\begin{split}
\hat{\boldsymbol{\theta}}(\boldsymbol{\lambda})
& =
\diag(K_j^{-1/2})
\left(
K^{(1/2)\top}
P_{X_T}^\top
K^{(1/2)} + \diag(\lambda_j I)
\right)^{-1}
K^{(1/2)\top}
P_{X_T}^\top
(I - \frac{1}{n}\boldsymbol{1} \boldsymbol{1}^\top)
\boldsymbol{y}_T
\end{split}
\label{eq:kkt_sobolev}
\end{align}
where $K^{(1/2)} = (K_1^{1/2} ... K_J^{1/2})$, $I$ is the $n_T\times n_T$ identity matrix, and $P_{X_T}^\top = I - X_T (X_T^\top X_T)^{-1} X_T^\top$.

To apply Theorem~\ref{thrm:train_val}, we need to characterize how $\hat{g}(\boldsymbol{\lambda})(\cdot)$ varies with $\boldsymbol{\lambda}$.
Since we have the closed form solution to \eqref{eq:kkt_sobolev}, we use it to directly bound the Lipschitz factor $C_\Lambda(\boldsymbol{x} | D^{(n_T)})$.
From \citet{green1993nonparametric}, we know that the value of the cubic $\hat{g}_j$ on the interval $[t_L, t_R]$ can be defined using its values and second derivatives at the ends of the interval.
Let $h = t_R - t_L$.
Then the value of the cubic
\begin{align}
\begin{split}
\hat{g}_{j,\perp}(x_j)
& = \hat{\alpha_{1j}} x_j
+ \frac{(x_j - t_L) \hat{g}_{j, \perp} (t_R) + (t_R - t) \hat{g}_{j, \perp}(t_L)}{h}\\
& \quad - \frac{1}{6}(x_j - t_L)(t_R - x_j) \left\{
\left(
1 + \frac{x_j - t_L}{h}
\right) \hat{g}''_{j, \perp}(t_R^+)
\left(
1 + \frac{t_R - x_j}{h}
\right) \hat{g}''_{j, \perp}(t_L^+)
\right \}.
\label{eq:interp}
\end{split}
\end{align}
Let $\hat{\boldsymbol{\gamma}}_j$ be the vector of second derivatives of $\hat{g}''_{j, \perp}$ for observations in the training data.
Since the fitted functions $\hat{g}_{j, \perp}$ must be natural cubic splines, $\hat{\boldsymbol{\gamma}}_j$ and $\hat{\boldsymbol{\theta}}_j$ have a linear relationship:
\begin{align}
\hat{\boldsymbol{\gamma}}_j & = R^{-1}_j Q^\top_j K_j \hat{\boldsymbol{\theta}}_j
\label{eq:second_deriv}
\end{align}
where the matrix $R_j$ is a banded diagonally dominant matrix and $Q_j$ is a banded negative-semi-definite matrix that depend on the covariates $x_j$ in the training data.
For the definitions of $R_j$ and $Q_j$, refer to \citet{green1993nonparametric}.
Let $h_j(D^{(n_T)})$ be the smallest distance between observations of the $j$th covariates in the training data $T$.
Then using the Gershgorin circle theorem \citep{gershgorin1931uber}, one can show that all the eigenvalues of $R_j$ are larger than $\frac{1}{3} h_j(D^{(n_T)})$ and all the eigenvalues of $Q_j$ have magnitudes no greater than $4/h_j(D^{(n_T)})$.
Thus using \eqref{eq:interp} and \eqref{eq:second_deriv}, we have that
\begin{align}
\left \|
\nabla_{\lambda} \hat{g}_{j, \perp}(\boldsymbol{\lambda})(x_j)
\right \|_2
\le
\frac{c}{h_j(D^{(n_T)})^2}
\left\|
\nabla_{\lambda} K_j \hat{\boldsymbol{\theta}}_j(\boldsymbol{\lambda})
\right \|_2
\end{align}
for some absolute constant $c > 0$.
To bound the second term on the right hand side, we know from \eqref{eq:kkt_sobolev} that
\begin{align}
& \nabla_{\lambda_\ell} K_j \hat{\boldsymbol{\theta}}_j(\boldsymbol{\lambda})\\
& = 
\left[
\begin{matrix}
0 & .. & 0 & K_j^{1/2} & 0 & .. & 0
\end{matrix}
\right]
\left(
K^{(1/2), \top} P_{X_T}^\top K^{(1/2)} + \diag\{\lambda_j I\}_{j=1:J}
\right)^{-2}
K^{(1/2), \top}
P_{X_T}^\top (I - \frac{1}{n} \boldsymbol{1} \boldsymbol{1}^\top) \boldsymbol{y}_T
\end{align}
if $\ell = j$.
Otherwise $\nabla_{\lambda_\ell} K_j \hat{\boldsymbol{\theta}}_j(\boldsymbol{\lambda}) = 0$.
Thus
\begin{align}
\left\|
\nabla_{\lambda_\ell} K_j \hat{\boldsymbol{\theta}}_j(\boldsymbol{\lambda})
\right \|_2
\le
\lambda_{\min}^{-2} \|\boldsymbol{y}_T\|_2 \sqrt{\|K_j\|_2 \sum_{{j'}=1}^J \|K_{j'}\|_2^2}
\end{align}
The eigenvalues of $K_j$ are bounded above by the largest row sum, which is no more than $2 n_T$ (assuming all training covariates are between 0 and 1).
Putting the results above together, we have
\begin{align}
\left \|
\nabla_{\lambda} \hat{g}_{j, \perp}(\boldsymbol{\lambda})(x_j)
\right \|_2
\le
\frac{c \sqrt{J} n_T}{h_j(D^{(n_T)})^2 \lambda_{\min}^2}
\|\boldsymbol{y}_T\|_2.
\end{align}
Also, we have from \eqref{eq:kkt_sobolev_linear} that
\begin{align}
\left \| \nabla_{\lambda} \hat{\boldsymbol{\alpha}_1}(\boldsymbol{\lambda}) \right \|_2
& =
\left \|
\left(
X_T^\top X_T
\right)^{-1}
X_T^\top
\nabla_{\lambda_j} K \hat{\boldsymbol{\theta}}(\boldsymbol{\lambda})
\right \|_2 \\
& =
\left \|
\left(
X_T^\top X_T
\right)^{-1}
X_T^\top
\nabla_{\lambda_j} K_j \hat{\boldsymbol{\theta}}_j(\boldsymbol{\lambda})
\right \|_2 \\
& \le
\left \|
\left(
X_T^\top X_T
\right)^{-1}
X_T^\top
\right \|_2
\lambda_{\min}^{-2} \|\boldsymbol{y}_T\|_2 n_T \sqrt{J}
\end{align}
Finally we can conclude that
\begin{align}
\begin{split}
\left \|
\hat{g}_j(\boldsymbol{\lambda}^{(1)})(x_j)
- \hat{g}_j(\boldsymbol{\lambda}^{(2)})(x_j)
\right \|_2
& \le
\left(
|x_j|
\left \|
\left(
X_T^\top X_T
\right)^{-1}
X_T^\top
\right \|_2
+
\frac{c }{h_j(D^{(n_T)})^2}
\right)\\
& \quad \times
\sqrt{J}
n_T
\lambda_{\min}^{-2}
\|\boldsymbol{y}_T\|_2
\| \boldsymbol{\lambda}^{(1)} - \boldsymbol{\lambda}^{(2)}\|_2
\end{split}
\label{eq:sob_add_comp}
\end{align}
By triangle inequality, we get the Lipschitz factor for the fitted model $\hat{g}$ by summing up \eqref{eq:sob_add_comp} for $j = 1,..,J$.
We find that the Lipschitz factor in \eqref{eq:param_add_lipschitz} is
\begin{align}
C_\Lambda(\boldsymbol{x} | T)
=
\left(
J
\left \|
\left(
X_T^\top X_T
\right)^{-1}
X_T^\top
\right \|_2
+
\sum_{j=1}^J \frac{c }{h_j(T)^2}
\right)
\sqrt{J}
n^{2t_{\min} + 1}
\|\boldsymbol{y}\|_T.
\label{eq:sobolev_lipschitz}
\end{align}

\noindent\textbf{Example~\ref{ex:elastic_net_tv}} (Multiple elastic nets, training-validation split)
Here we check that all the conditions for Lemma~\ref{lemma:nonsmooth} are satisfied.

First we check Condition~\ref{condn:nonsmooth1}.
Since the absolute value function $|\cdot|$ is twice-continuously differentiable everywhere except at zero, the directional derivatives of $||\boldsymbol \theta^{(j)}||_1$ at $\hat{\boldsymbol{\theta}}(\boldsymbol{\lambda})$ only exist along directions spanned by the columns of $\boldsymbol I_{I^{(j)}(\boldsymbol \lambda)}$.
Thus the penalized training loss $L_T(\cdot, \boldsymbol{\lambda})$ is twice differentiable with respect to the directions in
\begin{align}
\Omega^{L_T(\cdot, \lambda)}(\hat{\boldsymbol{\theta}}(\boldsymbol{\lambda} | T))
= \spann(I_{I^{(1)}(\boldsymbol \lambda)}) \times ... \times \spann(I_{I^{(J)}(\boldsymbol \lambda)}).
\label{eq:en_diff_space}
\end{align}

Moreover, the elastic net solution paths are piecewise linear \citep{zou2003regression}.
This implies that the nonzero indices of the elastic net estimates stay locally constant for almost every $\boldsymbol{\lambda}$; so \eqref{eq:en_diff_space} is also a local optimality space for $L_T(\cdot, \boldsymbol{\lambda})$.
In addition, this implies that Condition~\ref{condn:nonsmooth2} is satisfied.

We also check that the Hessian of the penalized training loss has a minimum eigenvalue bounded away from zero.
Consider the following orthogonal basis of \eqref{eq:en_diff_space} at $\hat{\boldsymbol{\theta}}(\boldsymbol{\lambda})$: $U(\boldsymbol{\lambda}) = \{U^{(j)}(\boldsymbol{\lambda})\}_{j = 1}^J$ where
\begin{align}
U^{(j)} =
\left(
\begin{matrix}
\boldsymbol{0} \\
I_{I^{(j)}(\boldsymbol \lambda)}\\
\boldsymbol{0}
\end{matrix}
\right)
\quad \forall j = 1,...,J.
\end{align}
The Hessian matrix of $L_T(\cdot, \boldsymbol{\lambda})$ with respect to directions $U(\boldsymbol{\lambda})$ is
\begin{align}
\boldsymbol U(\boldsymbol{\lambda})^\top \boldsymbol{X}_{T}^\top \boldsymbol{X}_{T} \boldsymbol U(\boldsymbol{\lambda}) + \lambda_1 w \boldsymbol{I}
\label{eq:en_hessian}
\end{align}
where $\boldsymbol{X}_{T} = (\boldsymbol{X}^{(1)} ... \boldsymbol{X}^{(J)})$
and $\boldsymbol{I}$ is the identity matrix with length equal to the number of nonzero elements in $\hat{\boldsymbol{\theta}}(\boldsymbol{\lambda})$.
Since the first summand is positive semi-definite and $\lambda_1 > \lambda_{\min}$, \eqref{eq:en_hessian} has a minimum eigenvalue of $\lambda_{\min}w$.

\noindent\textbf{Example~\ref{eq:elastic_net_cv}} (Multiple elastic nets, cross-validation)
Here we present details for establishing an oracle inequality when multiple elastic net penalties are tuned via the averaged version of $K$-fold cross-validation.
First we check the conditions in Theorem~\ref{thrm:kfold} are satisfied.
In the problem setup, $X$ is a log-concave vector and
$
\sup_{\|a\|_\infty = 1} \left\| X^\top a \right \|_{L_{\psi_2}} < c_R < \infty
$
for some constant $c_R$.
Using a similar procedure as \citet{lecue2012oracle}, we can then show that \eqref{eq:cv_assump1} and \eqref{eq:cv_assump2} in Assumption~\ref{assump:tail_margin} are satisfied with
$
K_0 \coloneqq 
(\|\boldsymbol{\theta}^*\|_\infty + K_0')
c_R$.

Next we find the Lipschitz factor.
We can upper bound the Lipschitz factor of the thresholded model with the Lipschitz factor of the un-thresholded model.
So Assumption~\ref{assump:lipschitz} is satisfied over $\mathbb{R}^p$ with
\begin{align}
C_\Lambda(\boldsymbol{x} | D^{(n_T)})
=
\frac{n^{2t_{\min}}}{w}
R^2 \sqrt{
	J p
	\left(
	\|\epsilon\|_{D^{(n_T)}}^{2}
	+\sum_{j=1}^J
	2 \|\boldsymbol{\theta}^{*,(j)}\|_1
	+ w\|\boldsymbol{\theta}^{*,(j)}\|_2^2
	\right)
}.
\label{eq:elastic_lipschitz_cv}
\end{align}

Finally, to apply Theorem~\ref{thrm:kfold}, we must find a bound for \eqref{eq:prob_bound_cv}.
Let $\sigma_0 = O_p(n^{4t_{\min}}R^4Jp/w^2 )$.
Using the fact that
$\left \| C_\Lambda(\cdot | D^{(n_T)}) \right \|_{L_{\psi_2}}^2$ is a linear function of $\|\epsilon\|_{D^{(n_T)}}^2$, which is a sub-exponential random variable, we have that
\begin{align}
\sum_{k = 1}^\infty k \Pr \left (
\| C_\Lambda(\cdot | D^{(n_T)}) \|_{L_{\psi_2}}
\ge 2^k \sigma_0
\right )
& \le
\sum_{k = 1}^\infty k \Pr \left (
\|\epsilon\|_{D^{(n_T)}}^2
\ge 2^{2k}
\right )
\le
c_1
\exp \left (
- \frac{c_0 n_T}{\|\epsilon\|^2_{L_{\psi_{2}}}}
\right )
\end{align}
for constants $c_0, c_1 > 0$.
Plugging in this bound to Theorem~\ref{thrm:kfold} gives us our desired result.

\bibliographystyle{unsrtnat}
\bibliography{hyperparam-theory}

\end{document}